\documentclass[shortpaper]{clv3}
\pdfoutput=1

\newcommand{\USETIKZ}{}

\usepackage{amsmath}

\let\numdef\relax
\usepackage{algorithmic}
\usepackage[linesnumbered,titlenumbered,ruled,vlined,resetcount,algosection]{algorithm2e}
\usepackage{xpatch}
\xpretocmd{\algorithmic}{}{}{}

\usepackage[usenames,dvipsnames]{xcolor}
\definecolor{darkblue}{rgb}{0, 0, 0.5}
\usepackage{booktabs}
\usepackage{float}

\usepackage[normalem]{ulem}

\usepackage{xparse}
\usepackage{xifthen} 

\usepackage[labelformat=simple]{subcaption}

\usepackage{todonotes}

\usepackage{hyperref}
\hypersetup{colorlinks=true,citecolor=darkblue, linkcolor=darkblue, urlcolor=darkblue}

\ifdefined\USETIKZ
\usepackage{tikz}
\fi

\DeclareMathOperator*{\OPTIM}{opt}
\newcommand \fixed {*}

\NewDocumentCommand \prob { m } 		{\mathbb{P}\left(#1\right)}

\NewDocumentCommand \eprob { m m }	{\mathbb{P}_{#1}\left(#2\right)}
\NewDocumentCommand \rprob { m }	{\eprob{\mathrm{pr}}{#1}}

\NewDocumentCommand \expe { m }  		{\mathbb{E}\left[#1\right]}

\NewDocumentCommand \eexpe { m m }	{\mathbb{E}_{#1}\left[#2\right]}
\NewDocumentCommand \eexpeapprox { m m }	{\tilde{\mathbb{E}}_{#1}\left[#2\right]}
\NewDocumentCommand \rexpe { m }	{\eexpe{\mathrm{pr}}{#1}}
\NewDocumentCommand \rexpeapprox { m }	{\eexpeapprox{\mathrm{pr}}{#1}}
\NewDocumentCommand \lexpe { m }	{\eexpe{\mathrm{pl}}{#1}}

\NewDocumentCommand \rcondexpe { m m }	{\rexpe{#1 \mid #2}}

\NewDocumentCommand \var { m }  	{\mathbb{V}\left[#1\right]}
\NewDocumentCommand \evar { m m }	{\mathbb{V}_{#1}\left[#2\right]}
\NewDocumentCommand \rvar { m }		{\evar{\mathrm{r}}{#1}}

\NewDocumentCommand \nats { } {\mathbb{N}}

\NewDocumentCommand \bigO { m } {O(#1)}

\NewDocumentCommand \numsamples { } {R}

\NewDocumentCommand \Sn { } {\mathcal{S}_n}
\NewDocumentCommand \Snh { } {\mathcal{S}_n^h}

\NewDocumentCommand \Snl { } {\mathcal{S}_n^l}
\NewDocumentCommand \Ln { } {\mathcal{L}_n}
\NewDocumentCommand \LnO { } {\mathcal{L}_n^0}
\NewDocumentCommand \Lnk { } {\mathcal{L}_n^k}
\NewDocumentCommand \Qn { } {\mathcal{Q}_n}
\NewDocumentCommand \Qnh { } {\mathcal{Q}_n^h}
\NewDocumentCommand \Qne { } {\mathcal{Q}_n^e}
\NewDocumentCommand \Qnb { } {\mathcal{Q}_n^b}
\NewDocumentCommand \Qnhl { } {\mathcal{Q}_n^{hl}}

\NewDocumentCommand \projective {} {\mathrm{pr}}

\NewDocumentCommand \Root { } {r}
\NewDocumentCommand \Pos { } {p}

\NewDocumentCommand \Ftree { O{T} }{ #1 }
\NewDocumentCommand \Rtree { O{\Root} O{\Ftree} }{ #2^{#1} }
\NewDocumentCommand \SubRtree { m O{\Root} O{\Ftree} }{ #3_{#1}^{#2} }

\NewDocumentCommand \Nvert { m } { n_{#1} }
\NewDocumentCommand \RNvert { } { \Nvert{\Root} }

\NewDocumentCommand \degree { m } {d_{#1}}
\NewDocumentCommand \Rdegree { } {\degree{\Root}}
\NewDocumentCommand \longdeg { m O{\Rtree} } { d_{#2}(#1) }

\NewDocumentCommand \neighs { m } {\Gamma_{#1}}
\NewDocumentCommand \Rneighs {} {\neighs{\Root}}
\NewDocumentCommand \longneighs { m O{\Rtree} } { \Gamma_{#2}(#1) }

\NewDocumentCommand \arr { } { \pi }
\NewDocumentCommand \invarr { } { \arr^{-1} }

\NewDocumentCommand \lenedge { m O{\arr} } {
	\ifthenelse{\isempty{#2}}
    {\delta_{#1}}
    {\delta_{#1}(#2)}
}
\NewDocumentCommand \lenedgep { m O{\arr} } { 
	\ifthenelse{\isempty{#2}}
    {\delta_{#1}^{\fixed}}
    {\delta_{#1}^{\fixed}(#2)}
}

\NewDocumentCommand \Dsymbol {} { D }
\NewDocumentCommand \D { m O{\arr} } {
	\Dsymbol_{#2}(#1)
}
\NewDocumentCommand \Dp { m O{\arr} } { 
	\Dsymbol_{#2}^{\fixed}(#1)
}

\NewDocumentCommand \anchor { m O{\arr} } {
	\ifthenelse{\isempty{#2}}
    {\alpha_{#1}}
    {\alpha_{#1}(#2)}
}
\NewDocumentCommand \coanchor { m O{\arr} } {
	\ifthenelse{\isempty{#2}}
    {\beta_{#1}}
    {\beta_{#1}(#2)}
}
\NewDocumentCommand \lengthsegment { m m O{\arr} } {
	\ifthenelse{\isempty{#3}}
    {\varphi_{#1}^{(#2)}}
    {\varphi_{#1,#3}(#2)}
}

\NewDocumentCommand \Unc { O{\Ftree} }{ \mathbf{P}(#1) }
\NewDocumentCommand \NUnc { O{\Ftree} }{ \mathbf{N}(#1) }
\NewDocumentCommand \Proj { O{\Rtree} }{ \mathbf{P_{\projective}}(#1) }
\NewDocumentCommand \NProj { O{\Rtree} }{ \mathbf{N_{\projective}}(#1) }
\NewDocumentCommand \DProj { O{\Rtree} }{ \mathbf{D_{\projective}}(#1) }
\NewDocumentCommand \DpProj { O{\Rtree} }{ \mathbf{D_{\projective}^{\fixed}}(#1) }

\NewDocumentCommand \PProj { O{\Pos} O{\Rtree} } { \mathbf{P_{\projective}}(#2;#1) }
\NewDocumentCommand \PNProj { O{\Pos} O{\Rtree} } { \mathbf{N_{\projective}}(#2;#1) }

\NewDocumentCommand \Vd { m } { \lenedge{#1}[] }

\NewDocumentCommand \VD { m } { \D{#1}[] }
\NewDocumentCommand \VDp { m } { \Dp{#1}[] }

\NewDocumentCommand \Vanchor { m } { \anchor{#1}[] }
\NewDocumentCommand \Vcoanchor { m } { \coanchor{#1}[] }
\NewDocumentCommand \Vlengthsegment { m m } { \lengthsegment{#1}{#2}[] }

\NewDocumentCommand \Xu { m O{} } {
	\ifthenelse{\isempty{#2}}
    {X_{#1}}
    {X_{#1}(#2)}
}
\NewDocumentCommand \Zk { m O{} } {
	\ifthenelse{\isempty{#2}}
    {Z_{#1}}
    {Z_{#1}(#2)}
}
\NewDocumentCommand \Sru { O{\Root} O{u} } { S(#1, #2) }
\NewDocumentCommand \SXru { O{\Root} O{u} } { S(\Xu{#1}, \Xu{#2}) }

\NewDocumentCommand \ExpeDProj { O{\Rtree} } { \rexpe{\VD{#1}} }
\NewDocumentCommand \ExpeDProjApprox { O{\Rtree} } { \rexpeapprox{\VD{#1}} }
\NewDocumentCommand \ExpeDpProj { O{\Rtree} } { \rexpe{\VDp{#1}} }
\NewDocumentCommand \ExpeDPlan { O{\Ftree} } { \lexpe{\VD{#1}} }
\NewDocumentCommand \VarDProj { O{\Rtree} } { \rvar{\VD{#1}} }
\NewDocumentCommand \ExpeDUnc { O{\Ftree} } { \expe{\VD{#1}} }
\NewDocumentCommand \ExpeDpUnc { O{\Ftree} } { \expe{\VDp{#1}} }
\NewDocumentCommand \VarDUnc { O{\Ftree} } { \var{\VD{#1}} }

\NewDocumentCommand \gDmin { m O{\Ftree} } { m_{#1}\left[ \Dsymbol(#2) \right] }
\NewDocumentCommand \Dmin { O{\Ftree} } { \gDmin{}[#1] }
\NewDocumentCommand \DminProj { O{\Rtree} } { \gDmin{\projective}[#1] }

\ifdefined\USETIKZ

\usetikzlibrary{arrows.meta,patterns}
\usetikzlibrary{ipe} 

\tikzstyle{ipe stylesheet} = [
  ipe import,
  even odd rule,
  line join=round,
  line cap=butt,
  ipe pen normal/.style={line width=0.4},
  ipe pen heavier/.style={line width=0.8},
  ipe pen fat/.style={line width=1.2},
  ipe pen ultrafat/.style={line width=2},
  ipe pen normal,
  ipe mark normal/.style={ipe mark scale=3},
  ipe mark large/.style={ipe mark scale=5},
  ipe mark small/.style={ipe mark scale=2},
  ipe mark tiny/.style={ipe mark scale=1.1},
  ipe mark normal,
  /pgf/arrow keys/.cd,
  ipe arrow normal/.style={scale=7},
  ipe arrow large/.style={scale=10},
  ipe arrow small/.style={scale=5},
  ipe arrow tiny/.style={scale=3},
  ipe arrow normal,
  /tikz/.cd,
  ipe arrows, 
  <->/.tip = ipe normal,
  ipe dash normal/.style={dash pattern=},
  ipe dash dotted/.style={dash pattern=on 1bp off 3bp},
  ipe dash dashed/.style={dash pattern=on 4bp off 4bp},
  ipe dash dash dotted/.style={dash pattern=on 4bp off 2bp on 1bp off 2bp},
  ipe dash dash dot dotted/.style={dash pattern=on 4bp off 2bp on 1bp off 2bp on 1bp off 2bp},
  ipe dash normal,
  ipe node/.append style={font=\normalsize},
  ipe stretch normal/.style={ipe node stretch=1},
  ipe stretch normal,
  ipe opacity 10/.style={opacity=0.1},
  ipe opacity 30/.style={opacity=0.3},
  ipe opacity 50/.style={opacity=0.5},
  ipe opacity 75/.style={opacity=0.75},
  ipe opacity opaque/.style={opacity=1},
  ipe opacity opaque,
]

\definecolor{red}{rgb}{1,0,0}
\definecolor{blue}{rgb}{0,0,1}
\definecolor{green}{rgb}{0,1,0}
\definecolor{yellow}{rgb}{1,1,0}
\definecolor{orange}{rgb}{1,0.647,0}
\definecolor{gold}{rgb}{1,0.843,0}
\definecolor{purple}{rgb}{0.627,0.125,0.941}
\definecolor{gray}{rgb}{0.745,0.745,0.745}
\definecolor{brown}{rgb}{0.647,0.165,0.165}
\definecolor{navy}{rgb}{0,0,0.502}
\definecolor{pink}{rgb}{1,0.753,0.796}
\definecolor{seagreen}{rgb}{0.18,0.545,0.341}
\definecolor{turquoise}{rgb}{0.251,0.878,0.816}
\definecolor{violet}{rgb}{0.933,0.51,0.933}
\definecolor{darkblue}{rgb}{0,0,0.545}
\definecolor{darkcyan}{rgb}{0,0.545,0.545}
\definecolor{darkgray}{rgb}{0.663,0.663,0.663}
\definecolor{darkgreen}{rgb}{0,0.392,0}
\definecolor{darkmagenta}{rgb}{0.545,0,0.545}
\definecolor{darkorange}{rgb}{1,0.549,0}
\definecolor{darkred}{rgb}{0.545,0,0}
\definecolor{lightblue}{rgb}{0.678,0.847,0.902}
\definecolor{lightcyan}{rgb}{0.878,1,1}
\definecolor{lightgray}{rgb}{0.827,0.827,0.827}
\definecolor{lightgreen}{rgb}{0.565,0.933,0.565}
\definecolor{lightyellow}{rgb}{1,1,0.878}
\definecolor{black}{rgb}{0,0,0}
\definecolor{white}{rgb}{1,1,1}

\fi

\issue{xx}{yy}{2022}
\runningtitle{The sum of edge lengths in random projective linearizations}
\runningauthor{Alemany-Puig \& Ferrer-i-Cancho}

\historydates{Action editor: Giorgo Satta; submission received: 7 July 2021; revised version received: 24 November 2021; accepted for publication: 18 January 2022}

\begin{document}
\allowdisplaybreaks

\title{Linear-time calculation of the expected sum of edge lengths in random projective linearizations of trees}

\author{
	Llu\'is Alemany-Puig
	\thanks{
		E-mail: lluis.alemany.puig@upc.edu.
	}
}
\affil{
	Universitat Polit\`ecnica de Catalunya, Barcelona, Catalonia, Spain.
}

\author{
	Ramon Ferrer-i-Cancho
	\thanks{
		E-mail: rferrericancho@cs.upc.edu.
	}
}
\affil{
	Universitat Polit\`ecnica de Catalunya, Barcelona, Catalonia, Spain.
}

\maketitle

\begin{abstract}
The syntactic structure of a sentence is often represented using syntactic dependency trees. The sum of the distances between syntactically related words has been in the limelight for the past decades. Research on dependency distances led to the formulation of the principle of dependency distance minimization whereby words in sentences are ordered so as to minimize that sum. Numerous random baselines have been defined to carry out related quantitative studies on languages. The simplest random baseline is the expected value of the sum in unconstrained random permutations of the words in the sentence, namely when all the shufflings of the words of a sentence are allowed and equally likely. Here we focus on a popular baseline: random projective permutations of the words of the sentence, that is, permutations where the syntactic dependency structure is projective, a formal constraint that sentences satisfy often in languages. Thus far, the expectation of the sum of dependency distances in random projective shufflings of a sentence has been estimated approximately with a Monte Carlo procedure whose cost is of the order of $\numsamples n$, where $n$ is the number of words of the sentence and $\numsamples$ is the number of samples; it is well known that the larger $\numsamples$, the lower the error of the estimation but the larger the time cost. Here we present formulae to compute that expectation without error in time of the order of $n$. Furthermore, we show that star trees maximize it, and give an algorithm to retrieve the trees that minimize it.
\end{abstract}

\section{Introduction}
\label{sec:introduction}

A successful way to represent the syntactic structure of a sentence is a dependency graph \cite{Nivre2006a} which relates the words of a sentence by pairing them with syntactic links as in Figure \ref{fig:example:dependency_tree}. Each link is directed and the arrow points from the {\em head word} to the {\em dependent word} (Figure \ref{fig:example:dependency_tree}). There are several conditions that are often imposed on the structure of dependency graphs \cite{Nivre2006a}. The first is {\em well-formedness}, namely, the graph is (weakly) connected. The second is {\em single-headedness}, that is, every word has at most one head. Another condition is {\em acyclicity}, that is, if two words, say $w_i$ and $w_j$, are connected via following one or more directed links from $w_i$ to $w_j$ then there is no path of directed links from $w_j$ to $w_i$. By definition, syntactic dependency trees always have a root vertex, that is, a vertex (word) with no head. The fourth condition is {\em projectivity}: often informally described as the situation where edges do not cross when drawn above the sentence and the root is not covered by any edge.

When a dependency graph is well-formed, single-headed and acyclic, the graph is a directed tree, called syntactic dependency tree \cite{Kuhlmann2006a,Gomez2011a}. In addition, a syntactic dependency structure is projective if, for every vertex $v$, all vertices reachable from $v$, that is, the yield of $v$, form a continuous substring within the linear ordering of the sentence \cite{Kuhlmann2006a}. Equivalently, a syntactic dependency structure is projective if the yield of each vertex of the tree forms a contiguous interval of positions in the linear ordering of the vertices. \citet{Kuhlmann2006a} define an interval (with endpoints $i$ and $j$) as the set $[i,j] = \{k \;|\; i\le k \text{ and } k\le j\}$.

A linear arrangement of a graph is planar if it does not have edge crossings \cite{Sleator1993a,Kuhlmann2006a}. Then projectivity can be characterized as a combination of two properties: planarity and the fact that the root is not covered \cite{Melcuk1988a}. Planarity was, to the best of our knowledge, first thought of as {\em one-page embeddings} of trees by \citet{Bernhart1974a}. Figure \ref{fig:example:dependency_tree} shows an example of a projective tree \ref{fig:example:dependency_tree}(a), a planar tree \ref{fig:example:dependency_tree}(b), and a non-planar tree \ref{fig:example:dependency_tree}(c) (see \citet{Bodirsky2005a} for further characterizations of syntactic dependency structures).

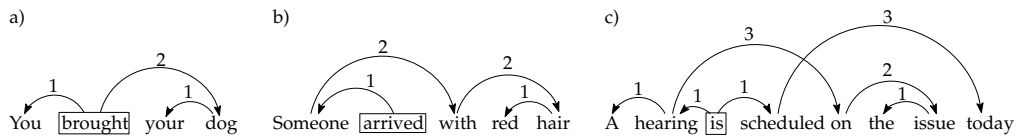
\begin{figure}
	\centering
\scalebox{0.71}{
\begin{tikzpicture}[ipe stylesheet]
  \node[ipe node]
     at (360, 752) {A};
  \node[ipe node]
     at (376, 752) {hearing
};
  \node[ipe node]
     at (416, 752) {is};
  \node[ipe node]
     at (432, 752) {scheduled};
  \node[ipe node]
     at (480, 752) {on};
  \node[ipe node]
     at (500, 752) {the};
  \node[ipe node]
     at (524, 752) {issue};
  \node[ipe node]
     at (552, 752) {today};
  \draw[-{ipe pointed[ipe arrow small]}]
    (392, 760)
     arc[start angle=26.5651, end angle=153.4349, radius=15.6525];
  \draw[-{ipe pointed[ipe arrow small]}]
    (416, 760)
     arc[start angle=45, end angle=135, radius=11.3137];
  \draw[-{ipe pointed[ipe arrow small]}]
    (420, 760)
     arc[start angle=-153.4349, end angle=-26.5651, x radius=15.6525, y radius=-15.6525];
  \draw[-{ipe pointed[ipe arrow small]}]
    (396, 760)
     arc[start angle=-169.6952, end angle=-10.3048, x radius=44.7214, y radius=-44.7214];
  \draw[-{ipe pointed[ipe arrow small]}]
    (452, 760)
     arc[start angle=-171.8931, end angle=-8.7462, x radius=54.5451, y radius=-54.5451];
  \draw[-{ipe pointed[ipe arrow small]}]
    (488, 760)
     arc[start angle=-161.5651, end angle=-18.4349, x radius=25.2982, y radius=-25.2982];
  \draw[-{ipe pointed[ipe arrow small]}]
    (532, 760)
     arc[start angle=33.6901, end angle=146.3099, radius=14.4222];
  \node[ipe node]
     at (184, 752) {Someone};
  \node[ipe node]
     at (232, 752) {arrived};
  \node[ipe node]
     at (272, 752) {with};
  \node[ipe node]
     at (300, 752) {red};
  \node[ipe node]
     at (324, 752) {hair};
  \draw[-{ipe pointed[ipe arrow small]}]
    (248, 760)
     arc[start angle=21.8014, end angle=158.1986, radius=21.5407];
  \draw[-{ipe pointed[ipe arrow small]}]
    (203.9996, 760.0001)
     arc[start angle=-167.4725, end angle=-12.6643, x radius=39.0306, y radius=-39.0306];
  \draw[-{ipe pointed[ipe arrow small]}]
    (281.6044, 760.2802)
     arc[start angle=-163.5349, end angle=-18.4349, x radius=28.3178, y radius=-28.3178];
  \draw[-{ipe pointed[ipe arrow small]}]
    (332, 760)
     arc[start angle=33.6901, end angle=153.4349, radius=14.4222];
  \node[ipe node]
     at (44, 752) {You};
  \node[ipe node]
     at (72, 752) {brought};
  \node[ipe node]
     at (116, 752) {your};
  \node[ipe node]
     at (148, 752) {dog};
  \draw[-{ipe pointed[ipe arrow small]}]
    (83.6369, 760.6594)
     arc[start angle=28.9531, end angle=153.4349, radius=17.8796];
  \draw[-{ipe pointed[ipe arrow small]}]
    (152, 760)
     arc[start angle=33.6901, end angle=153.4349, radius=14.4222];
  \draw[-{ipe pointed[ipe arrow small]}]
    (92.2363, 760.6032)
     arc[start angle=-164.56, end angle=-12.5288, x radius=32.599, y radius=-32.599];
  \draw
    (230, 760) rectangle (265, 750);
  \draw
    (70.3206, 760.691) rectangle (107.487, 749.297);
  \node[ipe node]
     at (375.43, 770.119) {1};
  \node[ipe node]
     at (406.96, 766.239) {1};
  \node[ipe node]
     at (431.699, 771.09) {1};
  \node[ipe node]
     at (233.385, 774.455) {1};
  \node[ipe node]
     at (315.121, 768.149) {1};
  \node[ipe node]
     at (136.557, 768.118) {1};
  \node[ipe node]
     at (515.618, 767.451) {1};
  \node[ipe node]
     at (507.614, 779.821) {2};
  \node[ipe node]
     at (306.147, 783.671) {2};
  \node[ipe node]
     at (122.005, 785.096) {2};
  \node[ipe node]
     at (65.008, 771.756) {1};
  \node[ipe node]
     at (239.934, 792.16) {2};
  \node[ipe node]
     at (434.124, 799.224) {3};
  \node[ipe node]
     at (507.856, 807.956) {3};
  \node[ipe node]
     at (44, 808) {a)};
  \node[ipe node]
     at (184, 808) {b)};
  \node[ipe node]
     at (360, 808) {c)};
  \draw
    (413.746, 760.453) rectangle (424.354, 749.013);
\end{tikzpicture}
}
	\caption{Examples of sentences and their syntactic dependency structure. Here arc labels indicate dependency distances. The word within a rectangle is the root of the sentence, and the number on top of each edge denotes its length. a) Projective dependency tree (adapted from \citet{Gross2009a}). b) Planar (but not projective) syntactic dependency structure (adapted from \citet{Gross2009a}). c) Non-projective syntactic dependency structure (adapted from \cite{Nivre2009a}).}
	\label{fig:example:dependency_tree}
\end{figure}

A free tree $\Ftree=(V,E)$ is an undirected acyclic graph (Figure \ref{fig:ftree_rtree_linarr}(a)), where $V$ is the set of vertices and $E$ is the set of edges. Here we represent the syntactic dependency structure of a sentence as a pair consisting of a rooted tree and a linear arrangement of its vertices. A rooted tree $\Rtree=(V,E;\Root)$ is a free tree $\Ftree=(V,E)$ with one of its vertices, say $\Root\in V$, labeled as its root and with the edges oriented from $\Root$ towards the leaves (Figure \ref{fig:ftree_rtree_linarr}(b)). A linear arrangement $\arr$ (also called {\em embedding}) of an $n$-vertex graph $G=(V,E)$ is a (bijective) function that assigns every vertex $u\in V$ to a position $\arr(u)$. Throughout this article, we use the terms `linear arrangement', `linear ordering', `arrangement', `linearization' interchangeably. In addition, we assume that $\arr(u)\in [n] = \{1,\cdots,n\}$. Linear arrangements are often seen as determined by the labeling of the vertices \cite{Chung1984a,Kuhlmann2006a}, but here we consider that the labeling of a graph and $\arr$ are independent. In order to clarify our notion of linear arrangement of a labeled graph $G$, we say that two linear arrangements $\arr_1$ and $\arr_2$ of $G$ are equal if and only if, for every vertex $u\in V$, it holds that $\arr_1(u)=\arr_2(u)$.

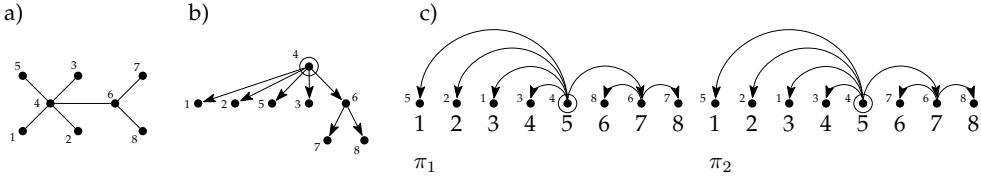
\begin{figure}
	\centering
\scalebox{0.87}{
\begin{tikzpicture}[ipe stylesheet]
  \pic
     at (100, 780) {ipe disk};
  \pic
     at (112, 768) {ipe disk};
  \pic
     at (140, 768) {ipe disk};
  \pic
     at (124, 756) {ipe disk};
  \pic
     at (100, 756) {ipe disk};
  \pic
     at (124, 780) {ipe disk};
  \pic
     at (152, 756) {ipe disk};
  \pic
     at (152, 780) {ipe disk};
  \draw
    (100, 780)
     -- (112, 768)
     -- (124, 780);
  \draw
    (100, 756)
     -- (112, 768)
     -- (124, 756);
  \draw
    (112, 768)
     -- (140, 768);
  \draw
    (140, 768)
     -- (152, 780);
  \draw
    (140, 768)
     -- (152, 756);
  \pic
     at (224, 784) {ipe disk};
  \pic
     at (176, 768) {ipe disk};
  \pic
     at (192, 768) {ipe disk};
  \pic
     at (208, 768) {ipe disk};
  \pic
     at (224, 768) {ipe disk};
  \pic
     at (240, 768) {ipe disk};
  \pic
     at (232, 752) {ipe disk};
  \pic
     at (248, 752) {ipe disk};
  \draw[-{ipe pointed[ipe arrow small]}]
    (224, 784)
     -- (178.3735, 768.7552);
  \draw[-{ipe pointed[ipe arrow small]}]
    (224, 784)
     -- (194.3945, 769.1249);
  \draw[-{ipe pointed[ipe arrow small]}]
    (224, 784)
     -- (209.7994, 769.3714);
  \draw[-{ipe pointed[ipe arrow small]}]
    (224, 784)
     -- (224.2184, 770.3573);
  \draw[-{ipe pointed[ipe arrow small]}]
    (224, 784)
     -- (238.1444, 769.6178);
  \draw[-{ipe pointed[ipe arrow small]}]
    (240, 768)
     -- (232.9683, 754.0897);
  \draw[-{ipe pointed[ipe arrow small]}]
    (240, 768)
     -- (247.0176, 754.213);
  \node[ipe node, font=\tiny]
     at (94.402, 751.52) {1};
  \node[ipe node, font=\tiny]
     at (118.722, 750.882) {2};
  \node[ipe node, font=\tiny]
     at (120.961, 783.039) {3};
  \node[ipe node, font=\tiny]
     at (104.963, 766.402) {4};
  \node[ipe node, font=\tiny]
     at (96.643, 782.559) {5};
  \node[ipe node, font=\tiny]
     at (136.801, 770.559) {6};
  \node[ipe node, font=\tiny]
     at (148.32, 782.239) {7};
  \node[ipe node, font=\tiny]
     at (147.202, 750.722) {8};
  \node[ipe node, font=\tiny]
     at (216.961, 787.68) {4};
  \node[ipe node, font=\tiny]
     at (170.079, 765.927) {1};
  \node[ipe node, font=\tiny]
     at (186.079, 765.927) {2};
  \node[ipe node, font=\tiny]
     at (202.079, 765.927) {5};
  \node[ipe node, font=\tiny]
     at (218.079, 765.927) {3};
  \node[ipe node, font=\tiny]
     at (242.719, 768.961) {6};
  \node[ipe node, font=\tiny]
     at (226.399, 747.205) {7};
  \node[ipe node, font=\tiny]
     at (243.517, 746.244) {8};
  \draw
    (224, 784) circle[radius=4];
  \pic
     at (272, 768) {ipe disk};
  \pic
     at (288, 768) {ipe disk};
  \pic
     at (304, 768) {ipe disk};
  \pic
     at (320, 768) {ipe disk};
  \pic
     at (336, 768) {ipe disk};
  \pic
     at (352, 768) {ipe disk};
  \pic
     at (368, 768) {ipe disk};
  \pic
     at (384, 768) {ipe disk};
  \node[ipe node]
     at (333.404, 756) {5};
  \node[ipe node]
     at (285.404, 756) {2};
  \node[ipe node]
     at (269.527, 756.123) {1};
  \node[ipe node]
     at (301.404, 756) {3};
  \node[ipe node]
     at (365.404, 756) {7};
  \node[ipe node]
     at (381.404, 756) {8};
  \node[ipe node]
     at (349.404, 756) {6};
  \node[ipe node]
     at (317.404, 756) {4};
  \draw
    (336, 768) circle[radius=4];
  \draw[-{ipe pointed[ipe arrow small]}]
    (336, 768)
     arc[start angle=0, end angle=162.7315, radius=8.1641];
  \draw[-{ipe pointed[ipe arrow small]}]
    (336, 768)
     arc[start angle=0, end angle=171.7404, radius=16.0121];
  \draw[-{ipe pointed[ipe arrow small]}]
    (336, 768)
     arc[start angle=0, end angle=173.9105, radius=24.0073];
  \draw[-{ipe pointed[ipe arrow small]}]
    (336, 768)
     arc[start angle=0, end angle=175.8767, radius=31.9915];
  \draw[-{ipe pointed[ipe arrow small]}]
    (336, 768)
     arc[start angle=180, end angle=350.9629, x radius=16.2137, y radius=-16.2137];
  \draw[-{ipe pointed[ipe arrow small]}]
    (368, 768)
     arc[start angle=-180, end angle=-16.3617, x radius=8.1657, y radius=-8.1657];
  \draw[-{ipe pointed[ipe arrow small]}]
    (368, 768)
     arc[start angle=0, end angle=163.6792, radius=8.1856];
  \node[ipe node]
     at (92, 804) {a)};
  \node[ipe node]
     at (172, 804) {b)};
  \node[ipe node]
     at (272, 804) {c)};
  \pic
     at (400, 768) {ipe disk};
  \pic
     at (416, 768) {ipe disk};
  \pic
     at (432, 768) {ipe disk};
  \pic
     at (448, 768) {ipe disk};
  \pic
     at (464, 768) {ipe disk};
  \pic
     at (480, 768) {ipe disk};
  \pic
     at (496, 768) {ipe disk};
  \pic
     at (512, 768) {ipe disk};
  \node[ipe node]
     at (397.404, 756) {1};
  \node[ipe node]
     at (413.404, 756) {2};
  \node[ipe node]
     at (461.404, 756) {5};
  \node[ipe node]
     at (429.404, 756) {3};
  \node[ipe node]
     at (493.404, 756) {7};
  \node[ipe node]
     at (509.404, 756) {8};
  \node[ipe node]
     at (477.404, 756) {6};
  \node[ipe node]
     at (445.404, 756) {4};
  \draw
    (464, 768) circle[radius=4];
  \node[ipe node]
     at (269.404, 740) {$\arr_1$};
  \node[ipe node]
     at (397.404, 740) {$\arr_2$};
  \node[ipe node, font=\tiny]
     at (298.323, 768.694) {1};
  \node[ipe node, font=\tiny]
     at (282.446, 768.817) {2};
  \node[ipe node, font=\tiny]
     at (265.46, 768.447) {5};
  \node[ipe node, font=\tiny]
     at (314.2, 768.57) {3};
  \node[ipe node, font=\tiny]
     at (377.953, 768.694) {7};
  \node[ipe node, font=\tiny]
     at (346.2, 768.447) {8};
  \node[ipe node, font=\tiny]
     at (362.077, 768.447) {6};
  \node[ipe node, font=\tiny]
     at (328.228, 768.694) {4};
  \node[ipe node, font=\tiny]
     at (426.323, 768.694) {1};
  \node[ipe node, font=\tiny]
     at (410.446, 768.817) {2};
  \node[ipe node, font=\tiny]
     at (393.46, 768.447) {5};
  \node[ipe node, font=\tiny]
     at (442.2, 768.57) {3};
  \node[ipe node, font=\tiny]
     at (473.953, 768.694) {7};
  \node[ipe node, font=\tiny]
     at (506.2, 768.447) {8};
  \node[ipe node, font=\tiny]
     at (490.077, 768.447) {6};
  \node[ipe node, font=\tiny]
     at (456.228, 768.694) {4};
  \draw[-{ipe pointed[ipe arrow small]}]
    (464, 768)
     arc[start angle=0, end angle=162.7315, radius=8.1641];
  \draw[-{ipe pointed[ipe arrow small]}]
    (464, 768)
     arc[start angle=0, end angle=171.7404, radius=16.0121];
  \draw[-{ipe pointed[ipe arrow small]}]
    (464, 768)
     arc[start angle=0, end angle=173.9105, radius=24.0073];
  \draw[-{ipe pointed[ipe arrow small]}]
    (464, 768)
     arc[start angle=0, end angle=175.8767, radius=31.9915];
  \draw[-{ipe pointed[ipe arrow small]}]
    (464, 768)
     arc[start angle=180, end angle=350.9629, x radius=16.2137, y radius=-16.2137];
  \draw[-{ipe pointed[ipe arrow small]}]
    (496, 768)
     arc[start angle=-180, end angle=-16.3617, x radius=8.1657, y radius=-8.1657];
  \draw[-{ipe pointed[ipe arrow small]}]
    (496, 768)
     arc[start angle=0, end angle=163.6792, radius=8.1856];
\end{tikzpicture}
}
	\caption{a) A free tree $\Ftree=(V,E)$. b) The tree $\Ftree$ rooted at $\Root=4$, yielding $\Rtree=(V,E;\Root)$ with $\Root=4$. c) Two different projective linear arrangements of $\Rtree$: $\arr_1(5)=\arr_2(5)=1$ (thus $\invarr_1(1)=\invarr_2(1)=5$); $\arr_1(8)=6$, $\arr_2(7)=6$.
	}
	\label{fig:ftree_rtree_linarr}
\end{figure}

In any linear arrangement $\arr$ of a graph $G=(V,E)$, one can define properties on the graph's edges and on the arrangement as a whole. The length of an edge between vertices $u$ and $v$ is their distance in the linear arrangement, usually defined as
\begin{equation}
\label{eq:edge_length}
\lenedge{uv} = |\arr(u) - \arr(v)|
\end{equation}
Thus, the length of an edge in the arrangement is the number of vertices between its endpoints plus one as in previous studies \cite{Iordanskii1987a, Shiloach1979a, Chung1984a, Hochberg2003a, Ferrer2004a, Gildea2007a, Gildea2010a, Ferrer2019a}. A less commonly used definition of edge length is \cite{Hudson1995a, Hiranuma1999a, Eppler2004a, Liu2017a}
\begin{equation}
\label{eq:edge_length_minus_1}
\lenedgep{uv} = |\arr(u) - \arr(v)| - 1
\end{equation}
Here we use $\D{G} = \sum_{uv\in E} \lenedge{uv}$ as the definition for the sum of edge lengths of $G$ when it is linearly arranged by $\arr$, but we also derive some results for $\Dp{G} = \sum_{uv\in E} \lenedgep{uv}$.

There exists sizable literature on the calculation of baselines for the sum of edge lengths on trees. These baselines are crucial for research on the Dependency Distance Minimization (DDm) principle \cite{Ferrer2004a,Liu2017a,Temperley2018a}. DDm was put forward by comparing actual dependency distances against a random baseline \cite{Ferrer2004a}. Concerning the computation of the minimum baseline, \citet{Iordanskii1987a}, and \citet{Hochberg2003a} independently devised an $O(n)$-time algorithm for planar (one-page) embedding of free trees. \citet{Gildea2007a} sketched an algorithm for projective embeddings of rooted trees. \citet{Alemany2022a} reviewed this problem and presented, to the best of our knowledge, the first $\bigO{n}$-time algorithm. Polynomial-time algorithms for unconstrained embeddings were presented by \citet{Shiloach1979a}, with complexity $\bigO{n^{2.2}}$, and later by \citet{Chung1984a}, with complexities $\bigO{n^2}$ and $\bigO{n^\lambda}$, where $\lambda$ is any real number satisfying $\lambda>\log 3/\log 2$. Concerning random baselines, the precursors are found in Z\"ornig's research on the distribution of the distance between repeats in a uniformly random arrangement of a sequence assuming that consecutive elements are at distance zero \cite{Zornig1984a} as in parallel research on syntactic dependency distances \cite{Hudson1995a, Hiranuma1999a, Eppler2004a, Liu2017a}. Later, \citet{Ferrer2004a,Ferrer2016a} studied the expectation of the random variable $\VD{\Ftree}$ defined as
\begin{equation}
\label{eq:sum_edge_lengths}
\VD{\Ftree} = \sum_{uv\in E}\Vd{uv}
\end{equation}
in uniformly random arrangements, where $\delta_{uv}$ is a random variable defined over uniformly random unconstrained linear arrangements of the tree $\Ftree$, resulting in
\begin{equation}
\label{eq:ExpeDUnconstrained:closed_form}
\ExpeDUnc = \frac{n^2 - 1}{3}
\end{equation}
Notice that $\ExpeDUnc$ does not depend on the topology of $\Ftree$.

While there are constant-time formulae for the expectation of $\VD{\Ftree}$ in unconstrained arrangements (Equation \ref{eq:ExpeDUnconstrained:closed_form}), a procedure to calculate the expected value of $\VD{\Ftree}$ under projectivity is not forthcoming. Our primary goal is to improve the calculation of the expected sum of edge lengths in uniformly random projective arrangements with respect to the Monte Carlo method or random sampling method put forward by \citet{Gildea2007a}. Such a widely used procedure \cite{Park2009a,Futrell2015a,Kramer2021a} estimates the expectation of $D$ of an $n$-vertex tree with an error that is negatively correlated with $\numsamples$, the amount of arrangements sampled, while its cost is directly proportional to that amount, that is $\bigO{\numsamples n}$. This raises the question of what would the minimum value of $\numsamples$ be to obtain accurate-enough estimations of the expectation of $\VD{\Ftree}$ in projective arrangements. In recent research \cite{Futrell2015a,Kramer2021a}, $\numsamples=10,\numsamples=100$ were used. Here we demonstrate that there is no need to answer this question since we provide formulae to calculate its exact value.

We improve upon these techniques by providing closed-form formulae for the expected value of $\VD{\Rtree}$ in uniformly random projective arrangements of $\Rtree$ that can be evaluated in $\bigO{n}$-time. More formally, our goal in this article is to find closed-form formulae for $\ExpeDProj$, the expectation of the random variable $\VD{\Rtree}$ conditioned to the set of projective arrangements, where the subscript `pr' indicates `projective linear arrangement'. Notice that $\ExpeDUnc$ in Equation \ref{eq:ExpeDUnconstrained:closed_form} has no subscript to indicate {\em unconstrained linear arrangement}. An unconstrained linear arrangement is one of the $n!$ possible orderings. $\ExpeDProj$ is a widely used random baseline for research on Dependency Distance Minimization \cite{Park2009a,Gildea2010a,Futrell2015a,Kramer2021a}.

The structure of this article is the following. We first derive, in Section \ref{sec:exp_D}, an arithmetic expression for $\ExpeDProj$, given by
\begin{theorem}
\label{thm:ExpeDProj}
Let $\Rtree=(V,E;\Root)$ be a tree rooted at $\Root\in V$. The expected sum of edge lengths $\VD{\Rtree}$ conditioned to uniformly random projective arrangements is
\begin{align}
\ExpeDProj
&=	\frac{\Rdegree(2\RNvert + 1) + \RNvert - 1}{6}
	+ \sum_{u\in\Rneighs} \ExpeDProj[\SubRtree{u}]
	\label{eq:ExpeDProj:recurrence} \\
&=	\frac{1}{6} \left( -1 + \sum_{v\in V} \Nvert{v} (2\degree{v} + 1) \right)
	\label{eq:ExpeDProj:closed_form}
\end{align}
where $\Nvert{u}$ denotes the number of vertices of the subtree $\SubRtree{u}$ rooted at $u\in V$, that is, $\Nvert{u}=|V(\SubRtree{u})|$, $\neighs{v}$ denotes the set of children of vertex $v$, and $\degree{v}=|\neighs{v}|$ is the out-degree of vertex $v$ in the rooted tree. If $\Rdegree=0$ then $\ExpeDProj=0$.
\end{theorem}

Section \ref{sec:maxima_minima} characterizes the class of trees that maximize $\ExpeDProj$, detailed in Theorem \ref{thm:ExpeDProj_maximised_by_Snh}.
\begin{theorem}
\label{thm:ExpeDProj_maximised_by_Snh}
For any $n$-vertex rooted tree $\Rtree$, we have that $\ExpeDProj \le \ExpeDProj[\Snh]$ with equality if, and only if, $\Rtree=\Snh$, where $\Snh$ denotes the star tree of $n$ vertices.
\end{theorem}
Then, a tight upper bound of $\ExpeDProj$ is given by $\ExpeDUnc$, as detailed in the next corollary.
\begin{corollary}
\label{cor:upper_bound_E_pr_D}
Given any $n$-vertex rooted tree $\Rtree=(V,E;\Root)$ rooted at $\Root\in V$, it holds that
\begin{equation}
\label{eq:tight_upper_bound_of_ExpeDProj}
\ExpeDProj \le \ExpeDProj[\Snh] = \ExpeDUnc = \frac{n^2 - 1}{3}
\end{equation}
where $\ExpeDUnc$ is the expected sum of edge lengths in uniformly random (unconstrained) linear arrangements (Equation \ref{eq:ExpeDUnconstrained:closed_form}) and $\Ftree$ is the free tree variant of $\Rtree$.
\end{corollary}
Theorem \ref{thm:ExpeDProj} and Corollary \ref{cor:upper_bound_E_pr_D} indicate that, for each $n$, a star tree rooted at its hub ($\Snh$) maximizes $\ExpeDProj$, achieving $(n^2 - 1)/3$. Section \ref{sec:maxima_minima} also shows that the minima can be calculated with a dynamic programming algorithm.

Section \ref{sec:exp_D:study_corpora} compares our new method to calculate $\ExpeDProj$ exactly against the Monte Carlo estimation method using dependency treebanks and find that commonly used values of $\numsamples$ can yield a large relative error in the estimation on a single tree. This new method is available in the Linear Arrangement Library \cite{Alemany2021d}. We finally present some conclusions and propose future work in Section \ref{sec:conclusions}.
\section{Expected sum of edge lengths}
\label{sec:exp_D}

We devote this section to characterize projective arrangements (Section \ref{sec:exp_D:N_projective}) and to derive an arithmetic expression to calculate the sum of expected edge lengths in said arrangements (Section \ref{sec:exp_D:exp_sum}). We end this section with some instantiations of said expression for particular classes of trees (Section \ref{sec:exp_D:tree_classes}).

\subsection{The number of random projective arrangements}
\label{sec:exp_D:N_projective}

The number of unconstrained arrangements of an $n$-vertex tree $\Ftree$ is $\NUnc=|\Unc|=n!$, where $\Unc$ denotes the set of all $n!$ arrangements of $\Ftree$, hence $\NUnc$ is independent from the tree structure. The number of projective arrangements of a tree, however, depends on its structure, in particular on the out-degree sequence of the tree, as is shown later in this section. Counting the amount of projective arrangements of a tree motivates a proper characterization that underpins the proof of Theorem \ref{thm:ExpeDProj}. For this, we need to introduce some notation.

Henceforth we denote directed edges of a rooted tree $\Rtree=(V,E;\Root)$ as $uv=(u,v)\in E$; all edges are oriented towards the leaves. We denote the set of children of a vertex $v\in V$ as $\neighs{v}$, and thus the out-degree of $v$ is $\degree{v}=|\neighs{v}|$ in the rooted tree. In particular, we refer to the root's children as $\Rneighs=\{u_1,\cdots,u_{\Rdegree}\} \subset V$. We denote the subtree of $\Rtree$ rooted at $u\in V$ as $\SubRtree{u}$; we denote its size (in vertices) as $\Nvert{u}=|V(\SubRtree{u})|$; notice that $\Nvert{u}\ge 1$. We say that $\SubRtree{v}$ is an immediate subtree of $\SubRtree{u}$ if $uv$ is an edge of the tree. Figure \ref{fig:rooted_tree} depicts a rooted tree and the immediate subtrees of $\Rtree$. 

We provide a closed-form formula for the number of projective arrangements of a rooted tree, $\NProj$. This result helps us characterize said arrangements.

\begin{figure}
	\centering
\scalebox{1}{
\begin{tikzpicture}[ipe stylesheet]
  \pic
     at (176, 760) {ipe disk};
  \pic
     at (116, 736) {ipe disk};
  \pic
     at (160, 736) {ipe disk};
  \pic
     at (216, 736) {ipe disk};
  \draw[ipe pointed-]
    (119.0502, 737.8193)
     -- (176, 760);
  \draw[ipe pointed-]
    (162.006, 738.4405)
     -- (176, 760);
  \draw[ipe pointed-]
    (213.5598, 738.1203)
     -- (176, 760);
  \draw
    (116, 736)
     -- (132, 692)
     -- (100, 692)
     -- (116, 736);
  \pic[ipe mark tiny]
     at (176, 716) {ipe disk};
  \pic[ipe mark tiny]
     at (188, 716) {ipe disk};
  \pic[ipe mark tiny]
     at (200, 716) {ipe disk};
  \node[ipe node]
     at (180, 760) {$\Root$};
  \node[ipe node]
     at (100, 736) {$u_1$};
  \node[ipe node]
     at (144, 736) {$u_2$};
  \node[ipe node]
     at (220, 740) {$u_{\Rdegree}$};
  \node[ipe node]
     at (108, 684) {$\SubRtree{u_1}$};
  \node[ipe node]
     at (152, 684) {$\SubRtree{u_2}$};
  \node[ipe node]
     at (204, 684) {$\SubRtree{u_{\Rdegree}}
$};
  \node[ipe node]
     at (156, 764) {$\Rtree$};
  \draw
    (160, 736)
     -- (176, 692)
     -- (144, 692)
     -- (160, 736);
  \draw
    (216, 736)
     -- (232, 692)
     -- (200, 692)
     -- (216, 736);
\end{tikzpicture}
}
	\caption{A tree $\Rtree$ rooted at $\Root$. The children of the root are $\Rneighs=\{u_1,u_2,\cdots,u_{\Rdegree}\}$, where $\Rdegree$ is the out-degree of $\Root$. Each $\SubRtree{u}$, for $u\in\Rneighs$, denotes the subtree of $\Rtree$ rooted at $u$.}
	\label{fig:rooted_tree}
\end{figure}
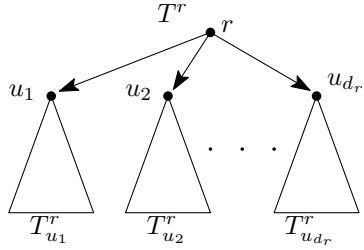

%
\begin{proposition}
\label{prop:NProj}
Let $\Rtree=(V,E;\Root)$ be a tree rooted at $\Root\in V$.
\begin{align}
\NProj
	&= (\Rdegree + 1)! \prod_{u\in\Rneighs} \NProj[\SubRtree{u}] \label{eq:NProj:recurrence} \\
	&= \prod_{v\in V} (\degree{v} + 1)! \label{eq:NProj:closed_form}
\end{align}
where $\degree{v}$ is the out-degree of vertex $v$ in the rooted tree. If $\Rdegree=0$ then $\NProj = 1$.
\end{proposition}
%

The fact that subtrees span over intervals \cite{Kuhlmann2006a} is central to the proof of Proposition \ref{prop:NProj}. Since intervals are associated to a fixed pair of starting and ending positions in a linear sequence, we use the term {\em segment} of a rooted tree $\SubRtree{u}$ to refer to a real segment within the linear ordering containing all vertices of $\SubRtree{u}$ (Figure \ref{fig:permutation_segments}); technically, that segment is an interval of length $\Nvert{u}$ whose starting and ending positions are unknown until the whole tree is fully linearized. Thus, a segment is a movable set of vertices within the linear ordering. The concept of {\em segment} is equivalent to the notion of continuous {\em constituent} in headed phrase structure representations \cite{Kuhlmann2006a}. Hereafter, for simplicity, we refer to the `segment of a tree in a linear arrangement' simply as `segment of a tree' assuming that such a segment is defined with respect to a linear arrangement.

%
\begin{proof}[Proof of Proposition \ref{prop:NProj}]
We can associate a set of segments to each vertex. The set of vertex $u$ contains $\degree{u}+1$ segments: one segment corresponds to $u$ (the only segment of length 1), and the remaining $\degree{u}$ segments correspond to the immediate subtrees of $\SubRtree{u}$. We obtain a projective linear arrangement by permuting the elements of each set for all vertices. Therefore, a projective arrangement can be seen as being recursively composed of permutations of sets of segments. Such `recursion' starts at the permutation of the set of segments associated to $\Root$. Note, then, that there are $(\degree{u} + 1)!$ possible permutations of the segments associated to vertex $u$. For a fixed permutation of the segments associated to $\Root$, there are $\prod_{u\in\Rneighs} \NProj[\SubRtree{u}]$ different projective arrangements of its immediate subtrees $\SubRtree{u}$, hence the recurrence in Equation \ref{eq:NProj:recurrence}. Equation \ref{eq:NProj:closed_form} follows upon unfolding the recurrence.
\end{proof}
%

The proof of Proposition \ref{prop:NProj} can be used to devise a simple procedure to generate projective arrangements uniformly at random, and another to enumerate all projective arrangements, of a rooted tree. As explained in previous articles \cite{Gildea2007a,Futrell2015a}, the former method consists of first generating a uniformly random permutation of the $\degree{v}+1$ segments associated to every vertex $v\in V$ and, afterwards, constructing the arrangement using these permutations. When a tree is linearized using the permutations of the sets of segments, we say that each segment becomes an interval.

\subsection{The expected sum of edge lengths in random arrangements}

We first review the problem of computing $\expe{\VD{\Ftree}}$: the expected value of $\VD{\Ftree}$ in uniformly random unconstrained arrangements so as to introduce the methodology applied for $\ExpeDProj$. The calculation requires two steps: first, the calculation of $\expe{\Vd{uv}}$, the expected length of an arbitrary edge joining vertices $u$ and $v$, and second, the calculation of $\expe{\VD{\Ftree}}$; henceforth we denote these values as $\expe{D}$ and $\expe{\Vd{}}$ since they only depend on the size of $\Ftree$, not on its topology. For simplicity, we assume the definition of edge length in Equation \ref{eq:edge_length}.

The calculation of $\expe{\Vd{}}$ requires the calculation of $\prob{\Vd{}}$, that is the probability that an edge linking two vertices has length $\Vd{}$. This is actually the proportion of unconstrained linear arrangements such that the two vertices are at distance $\Vd{}$. Since arrangements are unconstrained, said probability, and the corresponding expectation, do not depend on the edge. There are $\NUnc = n!$ unconstrained linear arrangements and $2(n-\Vd{})(n-2)!$ unconstrained arrangements where the pair of vertices are at distance $\Vd{}$, hence
\begin{equation}
\prob{\Vd{}}
    = \frac{2(n - \Vd{})(n-2)!}{n!}
    = \frac{2(n - \Vd{})}{n(n-1)}
\end{equation}
as expected from previous research \cite{Zornig1984a,Ferrer2004a}. Then, the expected length of an edge in an unconstrained random arrangement is \cite{Zornig1984a,Ferrer2004a}
\begin{align}
\expe{\Vd{}}
    &= \sum_{\Vd{} = 1}^{n - 1} \Vd{} \prob{\Vd{}} \\
	&= \frac{2}{n(n - 1)}\left( n\sum_{\Vd{} = 1}^{n - 1} \Vd{} - \sum_{\Vd{} = 1}^{n - 1} \Vd{}^2 \right) \\
    &= \frac{2}{n(n - 1)}\left( n \frac{1}{2}(n - 1)n - \frac{1}{6}(n - 1)n(2n - 1) \right)\\
    &= \frac{n + 1}{3}
\end{align}
The third equality follows from well-known formulae on the sum of integer numbers. The second step is the calculation of $\expe{D}$, the expected value of $D$ in an unconstrained arrangement. Since a tree has $n-1$ edges and applying linearity of expectation, $\expe{D}= (n-1)\expe{\Vd{}}$, which gives Equation \ref{eq:ExpeDUnconstrained:closed_form}. Note that neither $\expe{\Vd{uv}}$ nor $\ExpeDUnc$ depend on $\Ftree$'s topology (excluding $n$, the size of the tree).

\subsection{The expected sum of edge lengths in random projective arrangements}
\label{sec:exp_D:exp_sum}

To obtain an arithmetic expression for $\ExpeDProj$, we follow again a two-step approach.  First, we calculate the expected length of an edge in uniformly random projective arrangements. However, unlike the unconstrained case, the edge must be incident to the root. Second, we calculate the expected $\ExpeDProj$ applying the result of the first step. Before we proceed, we need to introduce some notation. 

An edge connecting the root of the tree ($\Root$) with one of its children ($u$) can be decomposed into two parts: its {\em anchor} \cite{Shiloach1979a,Chung1984a} and its {\em coanchor} (Figure \ref{fig:permutation_segments}). Such decomposition is also found in \citet{Gildea2007a,Park2009a} but using different terminology. In the context of projective linear arrangements, we define $\anchor{\Root u}$ as the length of the anchor, that is, the number of positions of the linear arrangement covered by the edge $\Root u$ in the segment of $\SubRtree{u}$ including the end of the edge $\arr(u)$; similarly, we define $\coanchor{\Root u}$ as the number of positions of the linear arrangement that are covered by that edge in segments other than that of $\SubRtree{u}$ and $\Root$. Put differently, $\anchor{\Root u}$ is the width of the part of $\SubRtree{u}$ covered by the edge $\Root u$ including the end of the edge $\arr(u)$; similarly, $\Vcoanchor{\Root u}$ is the total width of $\SubRtree{v}$ over all $\Root$'s children $v$ that fall between $\Root$ and $u$. Then the length of an edge connecting $\Root$ with $u$ is $\lenedge{\Root u} = |\arr(\Root) - \arr(u)| = \anchor{\Root u} + \coanchor{\Root u}$.

\begin{figure}
	\centering
\scalebox{1}{
\begin{tikzpicture}[ipe stylesheet]
  \pic
     at (260, 604) {ipe disk};
  \node[ipe node]
     at (258.384, 591.659) {$\Root$};
  \draw[shift={(384, 608)}, xscale=-1]
    (0, 0) rectangle (40, -8);
  \draw[shift={(312, 608)}, xscale=-1]
    (0, 0) rectangle (40, -8);
  \pic[ipe mark tiny]
     at (336, 604) {ipe disk};
  \pic[ipe mark tiny]
     at (328, 604) {ipe disk};
  \pic[ipe mark tiny]
     at (320, 604) {ipe disk};
  \node[ipe node]
     at (285.88, 588.221) {$\SubRtree{v}$};
  \node[ipe node]
     at (357.88, 588.221) {$\SubRtree{w}$};
  \node[ipe node]
     at (196, 644) {$\anchor{\Root u}$};
  \draw
    (224, 636)
     -- (224, 640)
     -- (260, 640)
     -- (260, 636);
  \draw
    (204, 636)
     -- (204, 640)
     -- (224, 640)
     -- (224, 636);
  \draw
    (184, 608) rectangle (224, 600);
  \pic[ipe mark tiny]
     at (248, 604) {ipe disk};
  \pic[ipe mark tiny]
     at (240, 604) {ipe disk};
  \pic[ipe mark tiny]
     at (232, 604) {ipe disk};
  \pic[ipe mark tiny]
     at (176, 604) {ipe disk};
  \pic[ipe mark tiny]
     at (168, 604) {ipe disk};
  \pic[ipe mark tiny]
     at (160, 604) {ipe disk};
  \node[ipe node]
     at (197.88, 588.221) {$\SubRtree{u}$};
  \draw
    (260.0004, 603.9998)
     arc[start angle=-146.6876, end angle=-45.9938, x radius=23.5227, y radius=-23.5227];
  \draw
    (259.9999, 604)
     arc[start angle=-157.8969, end angle=-26.5081, x radius=57.0994, y radius=-57.0994];
  \draw
    (259.9998, 603.9999)
     arc[start angle=17.6362, end angle=154.1922, radius=30.217];
  \node[ipe node]
     at (228, 644) {$\coanchor{\Root u}$};
  \draw
    (264, 608) rectangle (256, 600);
  \draw[ipe dash dotted]
    (224, 636)
     -- (224, 604);
  \draw[ipe dash dotted]
    (204, 636)
     -- (204, 608);
  \draw[ipe dash dotted]
    (260, 636)
     -- (260, 608);
\end{tikzpicture}
}
	\caption{A permutation $\tau$ of the segments associated to $\Root$. Each rectangle represents the segment of the root $\Root$ and those of the subtrees rooted at $u,v,w\in\Rneighs$, denoted as $\SubRtree{u},\SubRtree{v},\SubRtree{w}$. The representative vertices of the segments are, from left to right: $u$, $\Root$, $v$ and $w$. The anchor of edge $\Root u$ (whose length is $\Vanchor{\Root u}$), and the coanchor of edge $\Root u$ (whose length is $\Vcoanchor{\Root u}$), are delimited by the dotted lines above edge $\Root u$.}
	\label{fig:permutation_segments}
\end{figure}
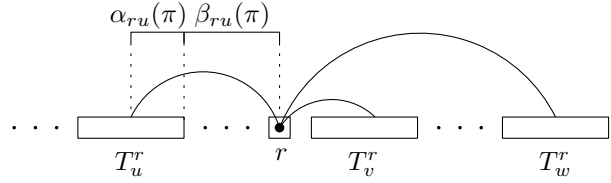

The next lemma shows that the expected value of $\Vd{\Root u}$ depends only on the size of the whole tree ($\RNvert$) and the size of the subtree rooted at the child ($\Nvert{u}$).
\begin{lemma}
\label{lemma:ExpeProj:length_anchor_coanchor_edge}
Let $\Rtree=(V,E;\Root)$ be a tree rooted at $\Root\in V$. Given an edge $\Root u\in E$, its anchor's expected length in uniformly random projective arrangements is
\begin{equation}
\label{eq:ExpeD:expected_length_anchor}
\rexpe{\Vanchor{\Root u}}= \frac{\Nvert{u} + 1}{2}
\end{equation}
while the expected length of its coanchor is
\begin{equation}
\label{eq:ExpeD:expected_length_coanchor}
\rexpe{\Vcoanchor{\Root u}}= \frac{\RNvert - \Nvert{u} - 1}{3}
\end{equation}
Therefore, the expected length of an edge $\Root u\in E$ in such arrangements is
\begin{equation}
\label{eq:ExpeD:root_to_child}
\rexpe{\Vd{\Root u}}
    = \rexpe{\Vanchor{\Root u} + \Vcoanchor{\Root u}}
    = \frac{2\RNvert + \Nvert{u} + 1}{6}
\end{equation}
\end{lemma}
\begin{proof}

By the law of total expectation 
\begin{equation}
\begin{split}
\rexpe{\Vanchor{\Root u}}
    &=
    \rcondexpe{\Vanchor{\Root u}}{\arr(u) < \arr(\Root)}
    \rprob{\arr(u) < \arr(\Root)} \\
    &+ \rcondexpe{\Vanchor{\Root u}}{\arr(u) > \arr(\Root)}
    \rprob{\arr(u) > \arr(\Root)}
\end{split}
\end{equation}
where $\rprob{\arr(u) < \arr(\Root)}$ is the probability that $u$ precedes $\Root$ in a random projective linear arrangement and $\rprob{\arr(u) > \arr(\Root)} + \rprob{\arr(u) < \arr(\Root)} = 1$. As any projective linear arrangement such that $u$ precedes $\Root$ has a reverse projective arrangement where $u$ follows $\Root$, then $\rprob{\arr(u) < \arr(\Root)} = 1/2$ and 
\begin{equation}
\rexpe{\Vanchor{\Root u}}
    =
    \frac{1}{2}
    \left(
        \rcondexpe{\Vanchor{\Root u}}{\arr(u) < \arr(\Root)} +
        \rcondexpe{\Vanchor{\Root u}}{\arr(u) > \arr(\Root)}
    \right)
\end{equation}
Now, let $q_u$ be the relative position of vertex $u$ in its segment in the (projective) linear arrangement (the $i$th vertex of said segment is at relative position $i$). If $u$ precedes $\Root$ in the linear arrangement ($\arr(u) < \arr(\Root)$) as in Figure \ref{fig:permutation_segments} then $\Vanchor{\Root u} = \Nvert{u} - q_u + 1$. If $u$ follows $\Root$ in the linear arrangement ($\arr(\Root) < \arr(u)$), $\Vanchor{\Root u} = q_u$. Applying these two results, we obtain
\begin{equation}
\rexpe{\Vanchor{\Root u}}
    = \frac{1}{2}
    \left(
        \rcondexpe{\Nvert{u} - q_u + 1}{\arr(u) < \arr(\Root)} +
        \rcondexpe{ q_u }{ \arr(u) > \arr(\Root)}
        \right)
\end{equation}
By symmetry, $\rcondexpe{q_u}{\arr(u) < \arr(\Root)} = \rcondexpe{q_u}{\arr(u) > \arr(\Root)}$ and then $\rexpe{\Vanchor{\Root u}} = (\Nvert{u} + 1)/2$, hence Equation \ref{eq:ExpeD:expected_length_anchor}.

In order to calculate $\Vcoanchor{\Root u}$, we define $s_{\Root u}$ as the number of intermediate segments between $\Root$ and the segment of $\SubRtree{u}$ in the linear arrangement. Therefore, $\Vcoanchor{\Root u}$ can be decomposed in terms of the lengths of each of these segments. The length of the $i$th segment in, say, the left-to-right order, is denoted as $\Vlengthsegment{\Root u}{i}$. Formally, $\Vcoanchor{\Root u}$ can be decomposed as
\begin{equation}
\Vcoanchor{\Root u} =
	\sum_{i=1}^{s_{\Root u}}
		\Vlengthsegment{\Root u}{i}
\end{equation}
By the law of total expectation,
\begin{equation}
\label{eq:ExpeDProj:expectation_beta:total_expected_law}
\rexpe{\Vcoanchor{\Root u}}
    = \sum_{s=1}^{\Rdegree - 1}
		\rcondexpe{\Vcoanchor{\Root u}}{s_{\Root u} = s} \rprob{s_{\Root u} = s}
\end{equation}
where $\rcondexpe{\Vcoanchor{\Root u}}{s_{\Root u}=s}$ is the expectation of $\Vcoanchor{\Root u}$ given that $u$ and $\Root$ are separated by $s$ segments, and $\rprob{s_{\Root u} = s}$ is the probability that $u$ and $\Root$ are separated by $s$ segments, both in uniformly random projective arrangements. On the one hand, 
\begin{equation}
\label{eq:ExpeDProj:length_coanchor:intermediate}
\rcondexpe{\Vcoanchor{\Root u}}{s_{\Root u} = s}
    = \rexpe{\sum_{i=1}^{s} \Vlengthsegment{\Root u}{i} }
    = s \rexpe{ \Vlengthsegment{\Root u}{i} }
\end{equation}
where
\begin{equation}
\rexpe{ \Vlengthsegment{\Root u}{i} } = \frac{\RNvert - \Nvert{u} - 1}{\Rdegree - 1} 
\end{equation}
is the average length of the segments excluding those of $\Root$ and $\SubRtree{u}$. On the other hand, $\rprob{s_{\Root u} = s}$ is the proportion of projective linear arrangements where the segments of $\SubRtree{u}$ and that of $\Root$ are separated by $s$ segments in the linear arrangement, that is   
\begin{equation}
\label{eq:probability_intermediate_segments}
\rprob{s_{\Root u} = s}
    =
    \frac
        {2(d_r - s)(d_r - 1)! \prod_{u\in\Rneighs} \NProj[\SubRtree{u}]}
        {(d_r + 1)! \prod_{u\in\Rneighs} \NProj[\SubRtree{u}]}
    = \frac{2(d_r - s)}{(d_r + 1)d_r}
\end{equation}
Plugging Equations \ref{eq:ExpeDProj:length_coanchor:intermediate} and \ref{eq:probability_intermediate_segments} into Equation \ref{eq:ExpeDProj:expectation_beta:total_expected_law}, 
one finally obtains Equation \ref{eq:ExpeD:expected_length_coanchor},
\begin{equation}
\rexpe{\Vcoanchor{\Root u}}
    = 2\frac{\RNvert - \Nvert{u} - 1}{(d_r+1)d_r(d_r - 1)} \sum_{s=1}^{d_r - 1} s(d_r - s)
    = \frac{\RNvert - \Nvert{u} - 1}{3}
\end{equation}
\end{proof}


Now we can derive an arithmetic expression for $\ExpeDProj$.

%
\begin{proof}[Proof of Theorem \ref{thm:ExpeDProj} (stated on page \pageref{thm:ExpeDProj})]
\label{proof:ExpeDProj}

Consider the random variable $\VD{\Rtree}$ over the probability space of uniformly random (unconstrained) linear arrangements of $\Rtree$, as defined above. This variable can be decomposed into two summations
\begin{equation}
\VD{\Rtree}
	= \sum_{u \in \Rneighs} \VD{\SubRtree{u}} +
	  \sum_{u \in \Rneighs} \Vd{\Root u}
\end{equation}
The first summation groups the edges by subtrees of $\Rtree$. The second summation groups the edges incident to the root $\Root$. Then, we can use linearity of expectation to obtain
\begin{equation}
\label{eq:expected_D_projective:definition}
\rexpe{\VD{\Rtree}}
	= \sum_{u \in \Rneighs} \ExpeDProj[\SubRtree{u}] +
	  \sum_{u \in \Rneighs} \rexpe{\Vd{\Root u}}
\end{equation}
The recurrence in Equation \ref{eq:ExpeDProj:recurrence} follows easily from applying Lemma \ref{lemma:ExpeProj:length_anchor_coanchor_edge} to Equation \ref{eq:expected_D_projective:definition}, which gives,
\begin{equation}
\sum_{\Root u \in E} \rexpe{\Vd{\Root u}}
	= \frac{1}{6}\sum_{\Root u \in E} (2\RNvert + 1) + \frac{1}{6}\sum_{\Root u \in E} \Nvert{u}
	= \frac{\Rdegree(2\RNvert + 1) + \RNvert - 1}{6}
\end{equation}
Equation \ref{eq:ExpeDProj:closed_form} follows upon unfolding the recurrence.
\end{proof}
%

In the proof above we implicitly use our definition of edge length $\lenedge{uv}$ (Equation \ref{eq:edge_length}). Nevertheless, the expression in Equation \ref{eq:ExpeDProj:closed_form} can be easily adjusted to use different definitions of edge length, for example, $\lenedgep{uv}$ (Equation \ref{eq:edge_length_minus_1}). It suffices to find an appropriate transformation of our definition of $\VD{\Rtree}$ (Equation \ref{eq:sum_edge_lengths}) into the one desired, namely $\VDp{\Rtree}$. The next corollary gives the solution. 

%
\begin{corollary}
\label{cor:ExpeDProj:reformulation}
Let $\Rtree=(V,E;\Root)$ be a tree rooted at $\Root\in V$. We have that
\begin{align}
\rexpe{\VDp{\Rtree}}
&=	\frac{\Rdegree(2\RNvert - 5) + \RNvert - 1}{6}
	+ \sum_{u\in\Rneighs} \rexpe{\VDp{\SubRtree{u}}}
	\label{eq:ExpeDProj:recurrence:n_minus_1} \\
&=	\frac{1}{6} \left( 5 - 6\RNvert + \sum_{v\in V} \Nvert{v} (2\degree{v} + 1) \right)
	\label{eq:ExpeDProj:closed_form:n_minus_1}
\end{align}
\end{corollary}
\begin{proof}[Proof of Corollary \ref{cor:ExpeDProj:reformulation}]
The fact that $\rexpe{\VDp{\Rtree}} = \ExpeDProj - (\RNvert - 1)$, transforms Equation \ref{eq:ExpeDProj:recurrence} into Equation \ref{eq:ExpeDProj:recurrence:n_minus_1} immediately as well as Equation \ref{eq:ExpeDProj:closed_form} into Equation \ref{eq:ExpeDProj:closed_form:n_minus_1} thanks to
\begin{equation}
\sum_{u\in\Rneighs} \ExpeDProj[\SubRtree{u}]
    = \sum_{u\in\Rneighs} \left( \rexpe{\VDp{\SubRtree{u}}}  + \Nvert{u} - 1 \right)
    = \RNvert - 1 - d_r + \sum_{u\in\Rneighs} 
 \rexpe{\VDp{\SubRtree{u}}}
\end{equation}
\end{proof}
%

It is easy to see that Equations \ref{eq:ExpeDProj:closed_form} and \ref{eq:ExpeDProj:closed_form:n_minus_1} can both be evaluated in $\bigO{n}$-time and $\bigO{n}$-space, where $n$ is the number of vertices of the tree: one only needs to compute the values $\Nvert{v}$ in $\bigO{n}$-time, store them in $\bigO{n}$ space, and then evaluate the formula, also in $\bigO{n}$-time using those values. In the analysis above, we are assuming that the values of $\degree{v}$ are already computed; depending on the data structure used to represent the tree, the cost of computing $\degree{v}$ might be relevant.

\subsection{Formulae for classes of trees}
\label{sec:exp_D:tree_classes}

Here we consider three kinds of free trees which are later transformed into rooted trees \cite{Harary1969a,Valiente2021a}. First, {\em linear} (or {\em path}) trees are trees in which the maximum degree is 2. {\em Star} trees consist of a vertex connected to $n-1$ leaves; also, a complete bipartite graph $K_{1,n-1}$. A {\em quasi-star} tree is a star tree in which one of its edges has been subdivided once with a vertex in the middle\footnote{Alternatively, an $n$-vertex quasi-star tree is obtained by joining to a 2-vertex complete graph, $K_2$, a pendant vertex to one end and $n-3$ pendant vertices to the other end of $K_2$; a quasi-star tree is a particular case of bistar tree \cite{SanDiego2014a}.}. For the following analyses, we define the hub of a rooted tree as the vertex of the underlying free tree that has the highest degree; we also use the term `leaf' to refer to a leaf in the underlying free tree. We now instantiate Equations \ref{eq:ExpeDProj:closed_form} and \ref{eq:NProj:closed_form} for several classes of trees (Table \ref{table:formulae_ExpeDProj_trees}): Star trees, $\Sn$, rooted at the hub, $\Snh$, and at a leaf, $\Snl$; Quasi-star trees, $\Qn$, rooted at the hub, $\Qnh$, at a leaf adjacent to the hub, $\Qnhl$, at the leaf not adjacent to the hub $\Qne$, and at the only internal vertex that is not the hub, $\Qnb$; Linear trees when rooted at a vertex at distance $k\ge 0$ from one of the leaves, $\Lnk$. Each class of tree is depicted in Figure \ref{fig:classes_of_trees}. These classes of trees are chosen for graph theoretic reasons. Linear trees minimize the variance of the degree \cite{Ferrer2013a}; star graphs maximize it \cite{Ferrer2013a} and all their unconstrained linear arrangements are planar; concerning $\Snh$, all its linear arrangements are projective; quasi star trees maximize the variance of the degree among trees whose set of unconstrained arrangements contains some non-planar arrangement \cite{Ferrer2016a}.

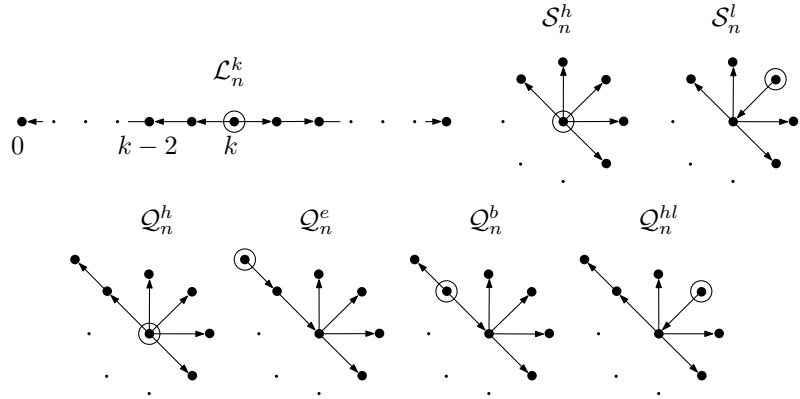
\begin{figure}
	\centering
\scalebox{1}{
\begin{tikzpicture}[ipe stylesheet]
  \pic
     at (204, 720) {ipe disk};
  \draw
    (204, 720) circle[radius=4];
  \pic
     at (220, 720) {ipe disk};
  \pic
     at (236, 720) {ipe disk};
  \pic
     at (284, 720) {ipe disk};
  \pic[ipe mark scale=1.0]
     at (248, 720) {ipe disk};
  \pic[ipe mark scale=1.0]
     at (260, 720) {ipe disk};
  \pic[ipe mark scale=1.0]
     at (272, 720) {ipe disk};
  \draw[-{>[ipe arrow tiny]}]
    (204, 720)
     -- (217.807, 720.048);
  \draw[-{>[ipe arrow tiny]}]
    (220, 720)
     -- (233.659, 720.115);
  \draw
    (236, 720)
     -- (244, 720);
  \draw[-{>[ipe arrow tiny]}]
    (276, 720)
     -- (281.846, 720.048);
  \node[ipe node]
     at (196, 736) {$\Lnk$};
  \pic
     at (188, 720) {ipe disk};
  \pic
     at (172, 720) {ipe disk};
  \pic
     at (124, 720) {ipe disk};
  \pic[ipe mark scale=1.0]
     at (160, 720) {ipe disk};
  \pic[ipe mark scale=1.0]
     at (148, 720) {ipe disk};
  \pic[ipe mark scale=1.0]
     at (136, 720) {ipe disk};
  \draw[-{>[ipe arrow tiny]}]
    (204, 720)
     -- (190.134, 720.015);
  \draw[-{>[ipe arrow tiny]}]
    (188, 720)
     -- (174.282, 720.015);
  \draw[shift={(172, 720)}, xscale=-1]
    (0, 0)
     -- (8, 0);
  \draw[-{>[ipe arrow tiny]}]
    (132, 720)
     -- (126.282, 720);
  \node[ipe node]
     at (120, 708) {$0$};
  \node[ipe node]
     at (160, 708) {$k-2$};
  \node[ipe node]
     at (200, 708) {$k$};
  \pic[ipe mark scale=3.0]
     at (328.027, 720) {ipe disk};
  \pic[ipe mark scale=3.0]
     at (350.745, 720.132) {ipe disk};
  \pic[ipe mark scale=3.0]
     at (344.108, 735.839) {ipe disk};
  \pic[ipe mark scale=3.0]
     at (327.848, 742.476) {ipe disk};
  \pic[ipe mark scale=3.0]
     at (344.219, 704.203) {ipe disk};
  \pic[ipe mark scale=1.0]
     at (327.958, 697.456) {ipe disk};
  \pic[ipe mark scale=3.0]
     at (312.141, 736.06) {ipe disk};
  \pic[ipe mark scale=1.0]
     at (305.282, 720.021) {ipe disk};
  \pic[ipe mark scale=1.0]
     at (311.919, 703.982) {ipe disk};
  \draw[-{>[ipe arrow tiny]}]
    (328, 720)
     -- (313.581, 734.574);
  \draw[-{>[ipe arrow tiny]}]
    (328, 720)
     -- (342.586, 734.338);
  \draw[-{>[ipe arrow tiny]}]
    (328, 720)
     -- (328.061, 740.453);
  \draw[-{>[ipe arrow tiny]}]
    (328, 720)
     -- (342.527, 705.629);
  \draw[-{>[ipe arrow tiny]}]
    (328, 720)
     -- (348.669, 720.081);
  \draw
    (328, 720) circle[radius=4];
  \node[ipe node]
     at (320, 756) {$\Snh$};
  \pic[ipe mark scale=3.0]
     at (392.027, 720) {ipe disk};
  \pic[ipe mark scale=3.0]
     at (414.745, 720.132) {ipe disk};
  \pic[ipe mark scale=3.0]
     at (408.108, 735.839) {ipe disk};
  \pic[ipe mark scale=3.0]
     at (391.848, 742.476) {ipe disk};
  \pic[ipe mark scale=3.0]
     at (408.219, 704.203) {ipe disk};
  \pic[ipe mark scale=1.0]
     at (391.958, 697.456) {ipe disk};
  \pic[ipe mark scale=3.0]
     at (376.141, 736.06) {ipe disk};
  \pic[ipe mark scale=1.0]
     at (369.282, 720.021) {ipe disk};
  \pic[ipe mark scale=1.0]
     at (375.919, 703.982) {ipe disk};
  \draw[{<[ipe arrow tiny]}-]
    (377.57, 734.35)
     -- (392, 720);
  \draw[{<[ipe arrow tiny]}-]
    (393.596, 721.475)
     -- (408, 736);
  \draw[-{>[ipe arrow tiny]}]
    (392, 720)
     -- (392.126, 740.322);
  \draw[-{>[ipe arrow tiny]}]
    (392, 720)
     -- (406.571, 705.531);
  \draw[-{>[ipe arrow tiny]}]
    (392, 720)
     -- (412.533, 719.959);
  \draw
    (408, 736) circle[radius=4];
  \node[ipe node]
     at (384, 756) {$\Snl$};
  \pic[ipe mark scale=3.0]
     at (172.027, 640) {ipe disk};
  \pic[ipe mark scale=3.0]
     at (194.745, 640.132) {ipe disk};
  \pic[ipe mark scale=3.0]
     at (188.108, 655.839) {ipe disk};
  \pic[ipe mark scale=3.0]
     at (171.848, 662.476) {ipe disk};
  \pic[ipe mark scale=3.0]
     at (188.219, 624.203) {ipe disk};
  \pic[ipe mark scale=1.0]
     at (171.958, 617.456) {ipe disk};
  \pic[ipe mark scale=3.0]
     at (156.141, 656.06) {ipe disk};
  \pic[ipe mark scale=1.0]
     at (149.282, 640.021) {ipe disk};
  \pic[ipe mark scale=1.0]
     at (155.919, 623.982) {ipe disk};
  \draw[-{>[ipe arrow tiny]}]
    (172, 640)
     -- (157.559, 654.499);
  \draw[-{>[ipe arrow tiny]}]
    (172, 640)
     -- (186.448, 654.243);
  \draw[-{>[ipe arrow tiny]}]
    (172, 640)
     -- (172.093, 660.267);
  \draw[-{>[ipe arrow tiny]}]
    (172, 640)
     -- (186.556, 625.393);
  \draw[-{>[ipe arrow tiny]}]
    (172, 640)
     -- (192.549, 639.965);
  \draw
    (172, 640) circle[radius=4];
  \node[ipe node]
     at (168, 680) {$\Qnh$};
  \node[ipe node]
     at (228, 680) {$\Qne$};
  \node[ipe node]
     at (292, 680) {$\Qnb$};
  \node[ipe node]
     at (356, 680) {$\Qnhl$};
  \pic[ipe mark scale=3.0]
     at (144, 668) {ipe disk};
  \draw[-{>[ipe arrow tiny]}]
    (156, 656)
     -- (145.613, 666.521);
  \pic[ipe mark scale=3.0]
     at (236.027, 640) {ipe disk};
  \pic[ipe mark scale=3.0]
     at (258.745, 640.132) {ipe disk};
  \pic[ipe mark scale=3.0]
     at (252.108, 655.839) {ipe disk};
  \pic[ipe mark scale=3.0]
     at (235.848, 662.476) {ipe disk};
  \pic[ipe mark scale=3.0]
     at (252.219, 624.203) {ipe disk};
  \pic[ipe mark scale=1.0]
     at (235.958, 617.456) {ipe disk};
  \pic[ipe mark scale=3.0]
     at (220.141, 656.06) {ipe disk};
  \pic[ipe mark scale=1.0]
     at (213.282, 640.021) {ipe disk};
  \pic[ipe mark scale=1.0]
     at (219.919, 623.982) {ipe disk};
  \draw[{<[ipe arrow tiny]}-]
    (234.429, 641.468)
     -- (220, 656);
  \draw[-{>[ipe arrow tiny]}]
    (236, 640)
     -- (250.527, 654.311);
  \draw[-{>[ipe arrow tiny]}]
    (236, 640)
     -- (235.968, 660.386);
  \draw[-{>[ipe arrow tiny]}]
    (236, 640)
     -- (250.502, 625.524);
  \draw[-{>[ipe arrow tiny]}]
    (236, 640)
     -- (256.628, 640.058);
  \draw
    (208, 668) circle[radius=4];
  \pic[ipe mark scale=3.0]
     at (208, 668) {ipe disk};
  \draw[{<[ipe arrow tiny]}-]
    (218.485, 657.361)
     -- (208, 668);
  \pic[ipe mark scale=3.0]
     at (300.027, 640) {ipe disk};
  \pic[ipe mark scale=3.0]
     at (322.745, 640.132) {ipe disk};
  \pic[ipe mark scale=3.0]
     at (316.108, 655.839) {ipe disk};
  \pic[ipe mark scale=3.0]
     at (299.848, 662.476) {ipe disk};
  \pic[ipe mark scale=3.0]
     at (316.219, 624.203) {ipe disk};
  \pic[ipe mark scale=1.0]
     at (299.958, 617.456) {ipe disk};
  \pic[ipe mark scale=3.0]
     at (284.141, 656.06) {ipe disk};
  \pic[ipe mark scale=1.0]
     at (277.282, 640.021) {ipe disk};
  \pic[ipe mark scale=1.0]
     at (283.919, 623.982) {ipe disk};
  \draw[{<[ipe arrow tiny]}-]
    (298.49, 641.409)
     -- (284, 656);
  \draw[-{>[ipe arrow tiny]}]
    (300, 640)
     -- (314.737, 654.308);
  \draw[-{>[ipe arrow tiny]}]
    (300, 640)
     -- (300.133, 660.441);
  \draw[-{>[ipe arrow tiny]}]
    (300, 640)
     -- (314.518, 625.465);
  \draw[-{>[ipe arrow tiny]}]
    (300, 640)
     -- (320.56, 639.897);
  \draw
    (284, 656) circle[radius=4];
  \pic[ipe mark scale=3.0]
     at (272, 668) {ipe disk};
  \draw[-{>[ipe arrow tiny]}]
    (284, 656)
     -- (273.548, 666.428);
  \pic[ipe mark scale=3.0]
     at (364.027, 640) {ipe disk};
  \pic[ipe mark scale=3.0]
     at (386.745, 640.132) {ipe disk};
  \pic[ipe mark scale=3.0]
     at (380.108, 655.839) {ipe disk};
  \pic[ipe mark scale=3.0]
     at (363.848, 662.476) {ipe disk};
  \pic[ipe mark scale=3.0]
     at (380.219, 624.203) {ipe disk};
  \pic[ipe mark scale=1.0]
     at (363.958, 617.456) {ipe disk};
  \pic[ipe mark scale=3.0]
     at (348.141, 656.06) {ipe disk};
  \pic[ipe mark scale=1.0]
     at (341.282, 640.021) {ipe disk};
  \pic[ipe mark scale=1.0]
     at (347.919, 623.982) {ipe disk};
  \draw[-{>[ipe arrow tiny]}]
    (364, 640)
     -- (349.512, 654.438);
  \draw[{<[ipe arrow tiny]}-]
    (365.475, 641.567)
     -- (380, 656);
  \draw[-{>[ipe arrow tiny]}]
    (364, 640)
     -- (363.957, 660.396);
  \draw[-{>[ipe arrow tiny]}]
    (364, 640)
     -- (378.571, 625.549);
  \draw[-{>[ipe arrow tiny]}]
    (364, 640)
     -- (384.36, 639.994);
  \draw
    (380, 656) circle[radius=4];
  \pic[ipe mark scale=3.0]
     at (336, 668) {ipe disk};
  \draw[-{>[ipe arrow tiny]}]
    (348, 656)
     -- (337.458, 666.524);
\end{tikzpicture}
}
	\caption{Linear trees ($\Ln$), star trees ($\Sn$) and quasi-star trees ($\Qn$) of $n$ vertices. Labels $0$ and $k$ in $\Lnk$ denote the distance of the labeled vertex from the same leaf. A circled dot marks a tree's root.}
	\label{fig:classes_of_trees}
\end{figure}

We choose linear trees to illustrate how one can instantiate Equation \ref{eq:ExpeDProj:closed_form}. In order to ease this task, we rewrite Equation \ref{eq:ExpeDProj:closed_form} using vectorial notation as
\begin{equation}
\label{eq:ExpeDProj:vectorial}
\ExpeDProj
    = \frac{1}{6}
    \left(
        -1 + 2\sum_{v\in V}\Nvert{v}\degree{v} + \sum_{v\in V} \Nvert{v}
    \right)
    = \frac{1}{6}
    \left(
        -1 + 2\overrightarrow{\Nvert{v}}\cdot\overrightarrow{\degree{v}} + \overrightarrow{\Nvert{v}}\cdot{\overrightarrow{1}}
    \right)
\end{equation}
For $\LnO$, we have
\begin{alignat*}{8}
&	\overrightarrow{\Nvert{v}}	&\quad=\quad&	(n,\; && n-1,\; && n-2,\; && \cdots,\; && 3,\; && 2,\; && 1) \\
&	\overrightarrow{\degree{v}}	&\quad=\quad&	(1,\; && 1,\;   && 1,\;   && \cdots,\; && 1,\; && 1,\; && 0)
\end{alignat*}
Then
\begin{equation}
\ExpeDProj[\LnO]
	= \frac{1}{6}\left( -1 + 2\sum_{i=2}^n i + \sum_{i=1}^n i \right)
	= \frac{(n - 1)(n + 2)}{4}
\end{equation}
For $\Lnk$ with $k>0$, we have
\begin{alignat*}{12}
&	\overrightarrow{\Nvert{v}}	&\quad=\quad&	(1,\; && 2,\; && \cdots,\; && k-1,\; && k,\; && n,\; && n-k-1,\; && n-k-2,\; && \cdots,\; && 2,\; && 1)\\
&	\overrightarrow{\degree{v}}	&\quad=\quad&	(0,\; && 1,\; && \cdots,\; && 1,\;   && 1,\; && 2,\; && 1,\;     && 1,\;     && \cdots,\; && 1,\; && 0)
\end{alignat*}
and then 
\begin{align}
\ExpeDProj[\Lnk]
	&= \frac{1}{6}
		\left[
			-1 +
			2\left( 2n + \sum_{j=2}^{k} j + \sum_{j=2}^{n-k-1} j \right)
			+
			\left( n + \sum_{j=1}^{k} j + \sum_{j=1}^{n-k-1} j \right)
		\right] \\
	&=
	\frac{(n - 1)(3n + 10) + 6k(k + 1 - n)}{12}
\end{align}
Regarding Equation \ref{eq:NProj:closed_form} for $\LnO$ and $\Lnk$ we have that
\begin{align}
\NProj[\LnO]
	&= (0 + 1)!\prod_{i=1}^{n-1} (1 + 1)!
	= 2^{n-1} \\
\NProj[\Lnk]
	&= (2+1)!(0+1)!(0+1)!\prod_{i=1}^{n-3} (1+1)!
	= 3 \cdot 2^{n-2}
\end{align}

\begin{table}
	\centering
	\caption{Instantiations of $\NProj$ and $\ExpeDProj$ for several classes of trees (Figure \ref{fig:classes_of_trees}).}
	\label{table:formulae_ExpeDProj_trees}
	\begin{tabular}{ccll}
		\toprule
		Class of tree	& $\Rtree$		& $\NProj$			& $\ExpeDProj$		\\
		
		\midrule
		
		Star 			& $\Snh$		& $n!$				& $(n^2 - 1)/3$ 	\\
						& $\Snl$ 		& $2(n-1)!$			& $n(2n - 1)/6$		\\
		
		\midrule
		
		Quasi star ($n\ge4$)
		                & $\Qnh$		& $2(n-1)!$			& $(2n^2 - 2n + 3)/6$ 	\\
						& $\Qne$		& $4(n-2)!$			& $(2n^2 - 2n + 3)/6$	\\
						& $\Qnb$		& $6(n-2)!$			& $(2n^2 - 3n + 7)/6$ 	\\
						& $\Qnhl$		& $4(n-2)!$			& $(2n^2 - 3n + 7)/6$ 	\\
		
		\midrule
		
		Linear 			& $\LnO$		& $2^{n-1}$			& $(n - 1)(n + 2)/4$	\\
						& $\Lnk$, $(k>0)$	& $3\cdot2^{n-2}$	& $[(n - 1)(3n + 10) + 6k(k + 1 - n)]/12$ \\
		
		\bottomrule
	\end{tabular}
\end{table}


\section{Maxima and minima}
\label{sec:maxima_minima}

In this section, we tackle the problem of computing the minima and characterizing the maxima of $\ExpeDProj$, both over all $n$-vertex root trees (keeping $n$ constant). In particular, we give a closed-form formula for the maximum value of $\ExpeDProj$ and characterize the trees that maximize it, as well as outline a dynamic programming algorithm to compute the minima. Henceforth, we use $\mathcal{T}_n$ to denote the set of $n$-vertex (unlabeled) rooted trees. Evidently, any tree that maximizes (resp. minimizes) $\ExpeDProj$ also maximizes (resp. minimizes) $\rexpe{\VDp{\Rtree}}$, thus we restrict our study to the former. Throughout this section, we use $\Nvert{i}$ to refer to the size of the subtree rooted at the $i$th child of the root for $1\le i\le \Rdegree$.

%


The construction of projective minimum linear arrangements has optimal substructure: optimal arrangements are composed of optimal arrangements of subtrees \cite{Hochberg2003a,Gildea2007a,Alemany2022a}. Similarly, the construction of $n$-vertex rooted trees $\Rtree\in\mathcal{T}_n$ that maximize (resp. minimize) $\ExpeDProj$ also has optimal substructure. The following lemma proves this claim.

\begin{lemma}
\label{lemma:optimal_substructure}
Let $\mathcal{T}_n$ be the set of all unlabeled rooted trees of $n$ vertices. $f(n)$, the optimal value of $\ExpeDProj$ satisfies
\begin{align}
	f(n)
	&=
	\OPTIM_{\Rtree\in\mathcal{T}_n} \{ \ExpeDProj \}
	\label{eq:optimal_rooted_trees:combinatorial_definition}
	\\
	&=
	\OPTIM_{0\le \Rdegree\le n - 1}
	\left\{
		\frac{\Rdegree(2n + 1) + n - 1}{6} +
		\OPTIM_{\substack{\Nvert{1} + \cdots + \Nvert{\Rdegree} = n - 1 \\ \Nvert{i}\ge 1 }}
		\left\{
			\sum_{i=1}^{\Rdegree} f(\Nvert{i})
		\right\}
	\right\}
	\label{eq:optimal_rooted_trees}
\end{align}
\end{lemma}
\begin{proof}
We can construct an $n$-vertex optimal tree using optimal subtrees. We can obtain the right hand side of Equation \ref{eq:optimal_rooted_trees} using the recurrence in Equation \ref{eq:ExpeDProj:recurrence}. An optimal $n$-vertex tree is one whose cost value is optimal among all optimal $n$-vertex trees whose root has fixed degree $\Rdegree\in\{1,\cdots,n-1\}$. Given a fixed root degree $\Rdegree$ such that $1\le \Rdegree\le n-1$, an optimal $n$-vertex rooted tree can be built by constructing an optimal $(n-1)$-vertex forest of $\Rdegree$ rooted trees, each of size $\Nvert{i}$ vertices and optimal among the $\Nvert{i}$-vertex trees. Choosing the $\Rdegree$ trees to be $\Nvert{i}$-optimal makes the sum of their costs, $\sum_i f(\Nvert{i})$, optimal.
\end{proof}
%

In subsequent paragraphs we use $f_M(n)$ to denote the {\em maximization} variant and $f_m(n)$ to denote the {\em minimization} variant of $f(n)$, respectively. Theorem \ref{thm:ExpeDProj_maximised_by_Snh} characterizes the maxima of $\ExpeDProj$. Perhaps not so surprisingly, the only maximum of $\ExpeDProj$ is obtained by $\Snh$.
%
\begin{proof}[Proof of Theorem \ref{thm:ExpeDProj_maximised_by_Snh} (stated on page \pageref{thm:ExpeDProj_maximised_by_Snh})]
We prove this by induction on $n$ using the formalization of the optimum in Equation \ref{eq:optimal_rooted_trees}. The base cases can be easily obtained by an exhaustive enumeration of the (unlabeled) rooted trees of $n$ vertices for some small $n$. Indeed, the only tree that maximizes $\ExpeDProj$ for $n\le2$ are the one-vertex tree and the two-vertex tree, which are both star trees.

In order to prove that $\ExpeDProj\le\ExpeDProj[\Snh]$ for $n\ge3$, it suffices to prove that
\begin{equation}
	\ExpeDProj[\Snh]
	=
	\frac{n^2 - 1}{3}
	>
	\max_{0\le \Rdegree\le n - 2}
	\left\{
		\frac{\Rdegree(2n + 1) + n - 1}{6} + L(n,\Rdegree)
	\right\}
\end{equation}
where
\begin{equation}
L(n,\Rdegree) =
	\max_{\substack{\Nvert{1} + \cdots + \Nvert{\Rdegree} = n - 1 \\ \Nvert{i}\ge 1 }}
		\left\{
			\sum_{i=1}^{\Rdegree} f_M(n_i)
		\right\}
\end{equation}
For this we need to know the maximum value of $L(n,\Rdegree)$. Applying the induction hypothesis (each maximum subtree of a tree of $n$ vertices is a star tree), we have that
\begin{equation}
L(n,\Rdegree)
	=
	\max_{\substack{\Nvert{1} + \cdots + \Nvert{\Rdegree} = n - 1 \\ \Nvert{i}\ge 1 }}
		\left\{ \sum_{i=1}^{\Rdegree} \frac{\Nvert{i}^2 - 1}{3} \right\} \\
	=
	\frac{1}{3}
	\left(
		- \Rdegree + 
		\max_{\substack{\Nvert{1} + \cdots + \Nvert{\Rdegree} = n - 1 \\ \Nvert{i}\ge 1 }}
			\left\{ \sum_{i=1}^{\Rdegree} \Nvert{i}^2 \right\}
	\right)
\end{equation}
Notice that any vector $(n_1, \cdots, n_i, \cdots, n_{\Rdegree})$, such that $1\le n_1 \le \cdots \le n_{\Rdegree}$, can be transformed into another vector $(n_1, \cdots, n_i - 1, \cdots, n_{\Rdegree} + 1)$, for any $n_i\ge 2$, such that the sum of squared components is strictly larger while the sum of the components remains constant. Therefore, the maximum sum of squares is obtained by choosing $n_{\Rdegree}=n-\Rdegree$ and $n_i=1$ for $1\le i< \Rdegree$, yielding
\begin{equation}
L(n,\Rdegree)
	= \frac{1}{3}\left[ -\Rdegree + (n - \Rdegree)^2 + (\Rdegree - 1) \right] \\
	= \frac{(n - \Rdegree)^2 - 1}{3}
\end{equation}
and the theorem holds if, and only if
\begin{equation}
\frac{n^2 - 1}{3}
	>
	\max_{0\le \Rdegree\le n - 2}
	\left\{
		\frac{\Rdegree(2n + 1) + n - 1}{6} + \frac{(n - \Rdegree)^2 - 1}{3}
	\right\}
\end{equation}
After rearranging the terms that do not depend on $d$ to the left-hand side we obtain
\begin{equation}
-n + 1
	>
	\max_{0\le \Rdegree\le n - 2}
	\left\{
		\Rdegree(2\Rdegree - 2n + 1)
	\right\}
\end{equation}
The right hand side of the inequality is maximized for $\Rdegree=n - 2$. Since this last inequality holds true when $n>5/2$, we are done.
\end{proof}
%

It is easy to see that $\ExpeDProj$ is bounded above by the expected value of $\VD{\Ftree}$ as stated in Corollary \ref{cor:upper_bound_E_pr_D} and justified in its proof below.

%
\begin{proof}[Proof of Corollary \ref{cor:upper_bound_E_pr_D} (stated on page \pageref{cor:upper_bound_E_pr_D})]
Due to Theorem \ref{thm:ExpeDProj_maximised_by_Snh} the maximum value of $\ExpeDProj$ is maximized by $\Snh$, formally $\ExpeDProj\le\ExpeDProj[\Snh]$, which, in turn, becomes $\ExpeDProj[\Snh]=(n^2 - 1)/3$, as shown in Table \ref{table:formulae_ExpeDProj_trees}. Finally, recall that $\ExpeDUnc=(n^2 - 1)/3$ (Equation \ref{eq:ExpeDUnconstrained:closed_form}), and thus $\ExpeDProj[\Snh]=\ExpeDUnc$.
\end{proof}
%

We devised a dynamic programming algorithm based on Lemma \ref{lemma:optimal_substructure} to calculate the distinct trees up to isomorphism minimizing $\ExpeDProj$. The method to obtain said values and trees is outlined in Algorithm \ref{alg:dynamic_programming}. That algorithm has two parameters: $n$, the number of vertices, and $H$ a hash table whose keys are natural numbers $k$ and the value associated to each key is a pair formed by $f_m(k)$ and the $k$-vertex trees $\Rtree$ that attain that $f_m(k)$. Notice that $n$ is an input parameter, while $H$ is an input/output parameter. In order to calculate the minimum $n$-vertex trees, the values of the parameters of the first call to Algorithm \ref{alg:dynamic_programming} are the value $n$ and an empty $H$. Now, Algorithm \ref{alg:dynamic_programming} is a direct evaluation of Equation \ref{eq:optimal_rooted_trees}, that is, for every value of out-degree of the root $\Rdegree$ ($1\le \Rdegree\le n-1$), it computes the value $f_m(n)$ by finding the partition of $n-1$ into $d$ summands that minimizes
\begin{equation}
\frac{\Rdegree(2n + 1) + n - 1}{6} +
	\OPTIM_{\substack{\Nvert{1} + \cdots + \Nvert{\Rdegree} = n - 1 \\ \Nvert{i}\ge 1 }}
	\left\{
		\sum_{i=1}^{\Rdegree} f_m(\Nvert{i})
	\right\}
\end{equation}
where $f_m(\Nvert{i})$ is calculated recursively and stored in the hash table $H$. Therefore, the algorithm's complexity heavily depends on the number of partitions of $n-1$, denoted as $p(n-1)$. The worst-case complexity of Algorithm \ref{alg:dynamic_programming}, then, is superpolynomial in $n$ due to the exponential nature of $p(n)$ \cite{Ramanujan1918a}. Finally, notice that the `Modified Cartesian' product in line \ref{algline:cartesian_product} of Algorithm \ref{alg:dynamic_programming} must ensure that no repeated trees are produced. Repeated trees arise in the standard Cartesian product due to the fact that some partitions may have repeated parts. As an example, consider the partition of $37$ with two repeated parts $(11,13,13)$; such parts have $1$ and $2$ non-isomorphic minimum trees respectively (Figure \ref{fig:minimum_trees}). Let $t_{11}=\{\Ftree_A\}$ be the unique $11$-vertex minimum tree, and $t_{13}=\{\Ftree_B,\Ftree_C\}$ be the two $13$-vertex minimum trees. The Cartesian product $t_{11}\times t_{13}\times t_{13}$ produces four forests, two of them being isomorphic: $(\Ftree_A,\Ftree_B,\Ftree_C)$ and $(\Ftree_A,\Ftree_C,\Ftree_B)$; the other two forests are $(\Ftree_A,\Ftree_B,\Ftree_B)$ and $(\Ftree_A,\Ftree_C,\Ftree_C)$. In order to obtain the unique $n$-vertex trees (up to isomorphism) that attain $f_m(n)$, we modify the standard Cartesian product. Let $\{T^*\}^{(i)}$ be the list of $n_i$-vertex minimum trees. Thus, one element in the Cartesian product is obtained by choosing the trees in the $j_1$th, $\cdots$, $j_{\Rdegree}$th positions of the lists, that is,
\begin{equation}
(\{T^*\}^{(1)}_{j_1},\cdots,\{T^*\}^{(\Rdegree)}_{j_{\Rdegree}})\in \{T^*\}^{(1)}\times\cdots\times \{T^*\}^{(\Rdegree)}
\end{equation}
where $1\le j_i\le |\{T^*\}^{(i)}|$ for all $i\in[1,\Rdegree]$. Now, tree uniqueness is ensured by forcing indices of every pair of lists $\{T^*\}^{(i)},\{T^*\}^{(i+1)}$ such that $n_i=n_{i+1}$ to be $j_i\le j_{i+1}$.

In this article, we do not characterize the minima of $\ExpeDProj$ since, unlike the amount of its maxima, which shows a clear regularity, the amount of its minima varies with $n$ in a non-monotonic fashion. Table \ref{table:minimum_trees} shows that the number of these minima oscillates between 1 and 2 for $n \leq 20$ and Figure \ref{fig:amount_minimum_trees} shows the amount of minimum trees in linear-log scale for $n\le178$. Figure \ref{fig:minimum_trees} suggests that the shape of the trees does not seem to fit into a simple class. Moreover, Figure \ref{fig:values_minimum} shows the values of $f_m(n)$ in log-linear scale. The straight line that is found in that scale for sufficiently large $n$ suggests $f_m(n) = O(n\log{n})$ asymptotic behavior to be confirmed in future research.

\begin{algorithm}
	\caption{Calculate the minimum value and trees of $\ExpeDProj$.}
	\label{alg:dynamic_programming}
	\DontPrintSemicolon
	
	\KwIn{$n\in\nats$, the number of vertices of the trees. $H$ is a hash table whose keys are natural numbers $k\in\nats$ and the value of each key is the pair formed by $f_m(k)$ and the $k$-vertex trees $\Rtree$ for which $\ExpeDProj=f_m(k)$.}
	\KwOut{The hash table $H$ updated to contain the results for $n'\le n$.}
	
	\SetKwProg{Fn}{Function}{ is}{end}
	\Fn{\textsc{Minimum\_E\_Projective}$(n, H)$} {
		\lIf {hash table $H$ contains key $n$} {
			Stop.
		}
	
		$f_m \gets (n^2 - 1)/3$ \tcp{{\small The minimum value initialized at a maximum}}
		$T^* \gets \emptyset$ \tcp{{\small The set of $n$-optimal trees}}
		\If {$n\le2$} {
			$T^*\gets\{ \text{the only rooted tree } \Rtree \text{ of } n \text{ vertices} \}$ \;
			$f_m\gets \ExpeDProj$ \;
		}
		\Else {
			\For {each $\Rdegree\in[1,n-1]$} {
				$\mathcal{P}\gets$ The set of partitions of `$n-1$' in $\Rdegree$ summands. \;
				\For {each partition $P=\{\Nvert{1},\cdots,\Nvert{\Rdegree}\}\in\mathcal{P}$} {
				
					\tcp{{\small Initialize the cost of the trees}}
					$C\gets (\Rdegree(2n + 1) + n - 1)/6$ \;
					
					\tcp{{\small Evaluate the partition recursively}}
					\For {each $\Nvert{i}\in P$} {
					    \tcp{{\small Calculate the $\Nvert{i}$-minimum trees}}
						\textsc{Minimum\_E\_Projective}$(\Nvert{i}, H)$\;
						
						$f_m^{(i)}, \{T^*\}^{(i)} \gets H[\Nvert{i}]$ \tcp{\small{Retrieve the trees from $H$}}
						
						$C\gets C + f_m^{(i)}$ \;
						\If {$C > f_m$} {
							Stop evaluating $P$ and move on to the next partition
						}
					}
					
					\tcp{{\small Modified Cartesian product of all $\{T^*\}^{(i)}$ (see text)}}
					$F\gets \{T^*\}^{(1)} \times \cdots \times \{T^*\}^{(d)}$ \label{algline:cartesian_product} \;
					
					\tcp{{\small $W$ is the set of $(n,d,P)$-optimal trees}}
					$W^*\gets$ add a root to every forest in $F$ \;
					
					\lIf {$C < f_m$} {
						$f_m\gets C$, $T^*\gets W^*$
					}
					\lElseIf {$C = f_m$} {
						$T^* \gets T^* \cup W^*$
					}
				}
			}
		}
		Store $\langle f_m, T^* \rangle$ in $H[n]$
	}
\end{algorithm}

\begin{table}
	\caption{The columns indicate, from left to right, the number of (unlabeled) rooted trees \cite{OEIS_UlabRootTrees}, the number of trees that minimize $\ExpeDProj$, and the value of $\ExpeDProj$ for such trees. The trees yielding these values are displayed in Figure \ref{fig:minimum_trees}. The horizontal line over a decimal digit denotes it is an infinitely repeating digit; thus $1/3=0.\overline{3}$.}
	\label{table:minimum_trees}
	\centering
	{\small
	\begin{tabular}[t]{rrcl}
	$n$ & \multicolumn{1}{c}{\# trees} & \# opt trees & $\ExpeDProj$ \\
	\midrule
	1 & 1 & 1 & $0$ \\
	2 & 1 & 1 & $1$ \\
	3 & 1 & 1 & $2.5$ \\
	4 & 4 & 2 & $4.5$ \\
	5 & 9 & 1 & $6.\overline{3}$ \\
	6 & 20 & 1 & $8.\overline{6}$ \\
	7 & 48 & 1 & $11$ \\
	8 & 115 & 2 & $13.8\overline{3}$ \\
	9 & 286 & 1 & $16.5$ \\
	10 & 719 & 2 & $19.\overline{3}$ \\
	\end{tabular}
	\quad
	\begin{tabular}[t]{rrcl}
	$n$ & \multicolumn{1}{c}{\# trees} & \# opt trees & $\ExpeDProj$ \\
	\midrule
	11 & 1,842 & 1 & $22$ \\
	12 & 4,766 & 1 & $25.1\overline{6}$ \\
	13 & 12,486 & 2 & $28.\overline{3}$ \\
	14 & 32,973 & 1 & $31.5$ \\
	15 & 87,811 & 1 & $34.\overline{6}$ \\
	16 & 235,381 & 1 & $38$ \\
	17 & 634,847 & 1 & $41.5$ \\
	18 & 1,721,159 & 2 & $45$ \\
	19 & 4,688,676 & 2 & $48.5$ \\
	20 & 12,826,228 & 2 & $52$ \\
	\end{tabular}
	}
\end{table}

\begin{figure}
	\centering
	\includegraphics[scale=0.69]{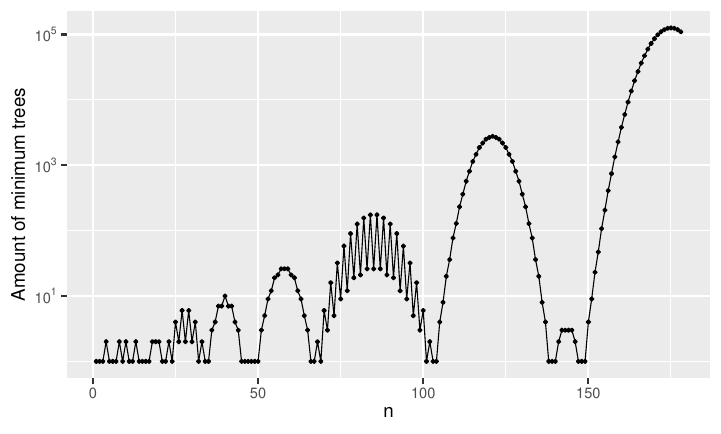}
	\caption{Amount of distinct $n$-vertex (unlabeled) trees $\Rtree$ that minimize $\ExpeDProj$ (that is, $\ExpeDProj=f_m(n)$) for $n\le 178$. Notice the logarithmic scale for the $y$-axis. The amount is computed via Algorithm \ref{alg:dynamic_programming}.}
	\label{fig:amount_minimum_trees}
\end{figure}

\begin{figure}
	\centering
\scalebox{0.69}{
\begin{tikzpicture}[ipe stylesheet]
  \pic
     at (48, 816) {ipe disk};
  \pic
     at (48, 800) {ipe disk};
  \pic
     at (48, 784) {ipe disk};
  \pic
     at (80, 816) {ipe disk};
  \pic
     at (80, 800) {ipe disk};
  \pic
     at (80, 784) {ipe disk};
  \pic
     at (80, 768) {ipe disk};
  \pic
     at (104, 816) {ipe disk};
  \pic
     at (96, 800) {ipe disk};
  \pic
     at (112, 800) {ipe disk};
  \pic
     at (112, 784) {ipe disk};
  \pic
     at (144, 784) {ipe disk};
  \pic
     at (144, 800) {ipe disk};
  \pic
     at (152, 816) {ipe disk};
  \pic
     at (160, 800) {ipe disk};
  \pic
     at (160, 784) {ipe disk};
  \draw
    (48, 816)
     -- (48, 800)
     -- (48, 784);
  \draw
    (80, 816)
     -- (80, 800)
     -- (80, 784)
     -- (80, 768);
  \draw
    (96, 800)
     -- (104, 816)
     -- (112, 800)
     -- (112, 784);
  \draw
    (144, 784)
     -- (144, 800)
     -- (152, 816)
     -- (160, 800)
     -- (160, 784);
  \pic
     at (192, 768) {ipe disk};
  \pic
     at (192, 784) {ipe disk};
  \pic
     at (192, 800) {ipe disk};
  \pic
     at (200, 816) {ipe disk};
  \pic
     at (208, 800) {ipe disk};
  \pic
     at (208, 784) {ipe disk};
  \pic
     at (240, 768) {ipe disk};
  \pic
     at (240, 784) {ipe disk};
  \pic
     at (240, 800) {ipe disk};
  \pic
     at (248, 816) {ipe disk};
  \pic
     at (256, 800) {ipe disk};
  \pic
     at (256, 784) {ipe disk};
  \pic
     at (256, 768) {ipe disk};
  \pic
     at (288, 752) {ipe disk};
  \pic
     at (288, 768) {ipe disk};
  \pic
     at (288, 784) {ipe disk};
  \pic
     at (288, 800) {ipe disk};
  \pic
     at (296, 816) {ipe disk};
  \pic
     at (304, 800) {ipe disk};
  \pic
     at (304, 784) {ipe disk};
  \pic
     at (304, 768) {ipe disk};
  \pic
     at (320, 768) {ipe disk};
  \pic
     at (320, 784) {ipe disk};
  \pic
     at (336, 784) {ipe disk};
  \pic
     at (328, 800) {ipe disk};
  \pic
     at (340, 816) {ipe disk};
  \pic
     at (352, 768) {ipe disk};
  \pic
     at (352, 784) {ipe disk};
  \pic
     at (352, 800) {ipe disk};
  \draw
    (192, 768)
     -- (192, 784)
     -- (192, 800)
     -- (200, 816)
     -- (208, 800)
     -- (208, 784);
  \draw
    (240, 768)
     -- (240, 784)
     -- (240, 800)
     -- (248, 816)
     -- (256, 800)
     -- (256, 784)
     -- (256, 768);
  \draw
    (288, 752)
     -- (288, 768)
     -- (288, 784)
     -- (288, 800)
     -- (296, 816)
     -- (304, 800)
     -- (304, 784)
     -- (304, 768);
  \pic
     at (384, 768) {ipe disk};
  \pic
     at (384, 784) {ipe disk};
  \pic
     at (400, 784) {ipe disk};
  \pic
     at (392, 800) {ipe disk};
  \pic
     at (404, 816) {ipe disk};
  \pic
     at (416, 768) {ipe disk};
  \pic
     at (416, 784) {ipe disk};
  \pic
     at (416, 800) {ipe disk};
  \pic
     at (400, 768) {ipe disk};
  \pic
     at (448, 768) {ipe disk};
  \pic
     at (448, 784) {ipe disk};
  \pic
     at (464, 784) {ipe disk};
  \pic
     at (456, 800) {ipe disk};
  \pic
     at (468, 816) {ipe disk};
  \pic
     at (480, 768) {ipe disk};
  \pic
     at (480, 784) {ipe disk};
  \pic
     at (480, 800) {ipe disk};
  \pic
     at (464, 768) {ipe disk};
  \pic
     at (480, 752) {ipe disk};
  \pic
     at (496, 768) {ipe disk};
  \pic
     at (496, 784) {ipe disk};
  \pic
     at (512, 784) {ipe disk};
  \pic
     at (504, 800) {ipe disk};
  \pic
     at (520, 816) {ipe disk};
  \pic
     at (528, 768) {ipe disk};
  \pic
     at (528, 784) {ipe disk};
  \pic
     at (536, 800) {ipe disk};
  \pic
     at (512, 768) {ipe disk};
  \pic
     at (544, 784) {ipe disk};
  \node[ipe node]
     at (36, 824) {$n=3$};
  \node[ipe node]
     at (84, 824) {$n=4$};
  \node[ipe node]
     at (140, 824) {$n=5$};
  \node[ipe node]
     at (188, 824) {$n=6$};
  \node[ipe node]
     at (236, 824) {$n=7$};
  \node[ipe node]
     at (308, 824) {$n=8$};
  \node[ipe node]
     at (388, 824) {$n=9$};
  \node[ipe node]
     at (476, 824) {$n=10$};
  \pic
     at (80, 680) {ipe disk};
  \pic
     at (80, 696) {ipe disk};
  \pic
     at (96, 696) {ipe disk};
  \pic
     at (88, 712) {ipe disk};
  \pic
     at (104, 728) {ipe disk};
  \pic
     at (112, 680) {ipe disk};
  \pic
     at (112, 696) {ipe disk};
  \pic
     at (120, 712) {ipe disk};
  \pic
     at (96, 680) {ipe disk};
  \pic
     at (128, 696) {ipe disk};
  \pic
     at (128, 680) {ipe disk};
  \node[ipe node]
     at (88, 736) {$n=11$};
  \pic
     at (160, 680) {ipe disk};
  \pic
     at (160, 696) {ipe disk};
  \pic
     at (176, 696) {ipe disk};
  \pic
     at (168, 712) {ipe disk};
  \pic
     at (184, 728) {ipe disk};
  \pic
     at (192, 680) {ipe disk};
  \pic
     at (192, 696) {ipe disk};
  \pic
     at (200, 712) {ipe disk};
  \pic
     at (176, 680) {ipe disk};
  \pic
     at (208, 696) {ipe disk};
  \pic
     at (208, 680) {ipe disk};
  \node[ipe node]
     at (168, 736) {$n=12$};
  \pic
     at (160, 664) {ipe disk};
  \pic
     at (240, 680) {ipe disk};
  \pic
     at (240, 696) {ipe disk};
  \pic
     at (256, 696) {ipe disk};
  \pic
     at (248, 712) {ipe disk};
  \pic
     at (264, 728) {ipe disk};
  \pic
     at (272, 680) {ipe disk};
  \pic
     at (272, 696) {ipe disk};
  \pic
     at (280, 712) {ipe disk};
  \pic
     at (256, 680) {ipe disk};
  \pic
     at (288, 696) {ipe disk};
  \pic
     at (288, 680) {ipe disk};
  \node[ipe node]
     at (280, 736) {$n=13$};
  \pic
     at (240, 664) {ipe disk};
  \pic
     at (256, 664) {ipe disk};
  \pic
     at (304, 680) {ipe disk};
  \pic
     at (304, 696) {ipe disk};
  \pic
     at (320, 696) {ipe disk};
  \pic
     at (312, 712) {ipe disk};
  \pic
     at (328, 728) {ipe disk};
  \pic
     at (336, 680) {ipe disk};
  \pic
     at (336, 696) {ipe disk};
  \pic
     at (344, 712) {ipe disk};
  \pic
     at (320, 680) {ipe disk};
  \pic
     at (352, 696) {ipe disk};
  \pic
     at (352, 680) {ipe disk};
  \pic
     at (304, 664) {ipe disk};
  \pic
     at (336, 664) {ipe disk};
  \pic
     at (384, 680) {ipe disk};
  \pic
     at (384, 696) {ipe disk};
  \pic
     at (400, 696) {ipe disk};
  \pic
     at (392, 712) {ipe disk};
  \pic
     at (408, 728) {ipe disk};
  \pic
     at (416, 680) {ipe disk};
  \pic
     at (416, 696) {ipe disk};
  \pic
     at (424, 712) {ipe disk};
  \pic
     at (400, 680) {ipe disk};
  \pic
     at (432, 696) {ipe disk};
  \pic
     at (432, 680) {ipe disk};
  \node[ipe node]
     at (392, 736) {$n=14$};
  \pic
     at (384, 664) {ipe disk};
  \pic
     at (400, 664) {ipe disk};
  \pic
     at (416, 664) {ipe disk};
  \pic
     at (464, 680) {ipe disk};
  \pic
     at (464, 696) {ipe disk};
  \pic
     at (480, 696) {ipe disk};
  \pic
     at (472, 712) {ipe disk};
  \pic
     at (488, 728) {ipe disk};
  \pic
     at (496, 680) {ipe disk};
  \pic
     at (496, 696) {ipe disk};
  \pic
     at (504, 712) {ipe disk};
  \pic
     at (480, 680) {ipe disk};
  \pic
     at (512, 696) {ipe disk};
  \pic
     at (512, 680) {ipe disk};
  \node[ipe node]
     at (472, 736) {$n=15$};
  \pic
     at (464, 664) {ipe disk};
  \pic
     at (480, 664) {ipe disk};
  \pic
     at (496, 664) {ipe disk};
  \pic
     at (512, 664) {ipe disk};
  \pic
     at (80, 592) {ipe disk};
  \pic
     at (80, 608) {ipe disk};
  \pic
     at (96, 608) {ipe disk};
  \pic
     at (88, 624) {ipe disk};
  \pic
     at (120, 640) {ipe disk};
  \pic
     at (112, 592) {ipe disk};
  \pic
     at (112, 608) {ipe disk};
  \pic
     at (120, 624) {ipe disk};
  \pic
     at (96, 592) {ipe disk};
  \pic
     at (128, 608) {ipe disk};
  \pic
     at (128, 592) {ipe disk};
  \node[ipe node]
     at (104, 648) {$n=16$};
  \pic
     at (144, 608) {ipe disk};
  \pic
     at (160, 608) {ipe disk};
  \pic
     at (144, 592) {ipe disk};
  \pic
     at (160, 592) {ipe disk};
  \pic
     at (152, 624) {ipe disk};
  \pic
     at (208, 592) {ipe disk};
  \pic
     at (208, 608) {ipe disk};
  \pic
     at (224, 608) {ipe disk};
  \pic
     at (216, 624) {ipe disk};
  \pic
     at (248, 640) {ipe disk};
  \pic
     at (240, 592) {ipe disk};
  \pic
     at (240, 608) {ipe disk};
  \pic
     at (248, 624) {ipe disk};
  \pic
     at (224, 592) {ipe disk};
  \pic
     at (256, 608) {ipe disk};
  \pic
     at (256, 592) {ipe disk};
  \node[ipe node]
     at (232, 648) {$n=17$};
  \pic
     at (272, 608) {ipe disk};
  \pic
     at (288, 608) {ipe disk};
  \pic
     at (272, 592) {ipe disk};
  \pic
     at (288, 592) {ipe disk};
  \pic
     at (280, 624) {ipe disk};
  \pic
     at (208, 576) {ipe disk};
  \pic
     at (320, 592) {ipe disk};
  \pic
     at (320, 608) {ipe disk};
  \pic
     at (336, 608) {ipe disk};
  \pic
     at (328, 624) {ipe disk};
  \pic
     at (360, 640) {ipe disk};
  \pic
     at (352, 592) {ipe disk};
  \pic
     at (352, 608) {ipe disk};
  \pic
     at (360, 624) {ipe disk};
  \pic
     at (336, 592) {ipe disk};
  \pic
     at (368, 608) {ipe disk};
  \pic
     at (368, 592) {ipe disk};
  \node[ipe node]
     at (388, 648) {$n=18$};
  \pic
     at (384, 608) {ipe disk};
  \pic
     at (400, 608) {ipe disk};
  \pic
     at (384, 592) {ipe disk};
  \pic
     at (400, 592) {ipe disk};
  \pic
     at (392, 624) {ipe disk};
  \pic
     at (320, 576) {ipe disk};
  \pic
     at (336, 576) {ipe disk};
  \pic
     at (416, 592) {ipe disk};
  \pic
     at (416, 608) {ipe disk};
  \pic
     at (432, 608) {ipe disk};
  \pic
     at (424, 624) {ipe disk};
  \pic
     at (456, 640) {ipe disk};
  \pic
     at (448, 592) {ipe disk};
  \pic
     at (448, 608) {ipe disk};
  \pic
     at (456, 624) {ipe disk};
  \pic
     at (432, 592) {ipe disk};
  \pic
     at (464, 608) {ipe disk};
  \pic
     at (464, 592) {ipe disk};
  \pic
     at (480, 608) {ipe disk};
  \pic
     at (496, 608) {ipe disk};
  \pic
     at (480, 592) {ipe disk};
  \pic
     at (496, 592) {ipe disk};
  \pic
     at (488, 624) {ipe disk};
  \pic
     at (416, 576) {ipe disk};
  \pic
     at (448, 576) {ipe disk};
  \pic
     at (96, 504) {ipe disk};
  \pic
     at (96, 520) {ipe disk};
  \pic
     at (112, 520) {ipe disk};
  \pic
     at (104, 536) {ipe disk};
  \pic
     at (136, 552) {ipe disk};
  \pic
     at (128, 504) {ipe disk};
  \pic
     at (128, 520) {ipe disk};
  \pic
     at (136, 536) {ipe disk};
  \pic
     at (112, 504) {ipe disk};
  \pic
     at (144, 520) {ipe disk};
  \pic
     at (144, 504) {ipe disk};
  \node[ipe node]
     at (168, 560) {$n=19$};
  \pic
     at (160, 520) {ipe disk};
  \pic
     at (176, 520) {ipe disk};
  \pic
     at (160, 504) {ipe disk};
  \pic
     at (176, 504) {ipe disk};
  \pic
     at (168, 536) {ipe disk};
  \pic
     at (96, 488) {ipe disk};
  \pic
     at (112, 488) {ipe disk};
  \pic
     at (192, 504) {ipe disk};
  \pic
     at (192, 520) {ipe disk};
  \pic
     at (208, 520) {ipe disk};
  \pic
     at (200, 536) {ipe disk};
  \pic
     at (232, 552) {ipe disk};
  \pic
     at (224, 504) {ipe disk};
  \pic
     at (224, 520) {ipe disk};
  \pic
     at (232, 536) {ipe disk};
  \pic
     at (208, 504) {ipe disk};
  \pic
     at (240, 520) {ipe disk};
  \pic
     at (240, 504) {ipe disk};
  \pic
     at (256, 520) {ipe disk};
  \pic
     at (272, 520) {ipe disk};
  \pic
     at (256, 504) {ipe disk};
  \pic
     at (272, 504) {ipe disk};
  \pic
     at (264, 536) {ipe disk};
  \pic
     at (192, 488) {ipe disk};
  \pic
     at (224, 488) {ipe disk};
  \pic
     at (128, 488) {ipe disk};
  \pic
     at (256, 488) {ipe disk};
  \pic
     at (304, 504) {ipe disk};
  \pic
     at (304, 520) {ipe disk};
  \pic
     at (320, 520) {ipe disk};
  \pic
     at (312, 536) {ipe disk};
  \pic
     at (344, 552) {ipe disk};
  \pic
     at (336, 504) {ipe disk};
  \pic
     at (336, 520) {ipe disk};
  \pic
     at (344, 536) {ipe disk};
  \pic
     at (320, 504) {ipe disk};
  \pic
     at (352, 520) {ipe disk};
  \pic
     at (352, 504) {ipe disk};
  \node[ipe node]
     at (376, 560) {$n=20$};
  \pic
     at (368, 520) {ipe disk};
  \pic
     at (384, 520) {ipe disk};
  \pic
     at (368, 504) {ipe disk};
  \pic
     at (384, 504) {ipe disk};
  \pic
     at (376, 536) {ipe disk};
  \pic
     at (304, 488) {ipe disk};
  \pic
     at (320, 488) {ipe disk};
  \pic
     at (400, 504) {ipe disk};
  \pic
     at (400, 520) {ipe disk};
  \pic
     at (416, 520) {ipe disk};
  \pic
     at (408, 536) {ipe disk};
  \pic
     at (440, 552) {ipe disk};
  \pic
     at (432, 504) {ipe disk};
  \pic
     at (432, 520) {ipe disk};
  \pic
     at (440, 536) {ipe disk};
  \pic
     at (416, 504) {ipe disk};
  \pic
     at (448, 520) {ipe disk};
  \pic
     at (448, 504) {ipe disk};
  \pic
     at (464, 520) {ipe disk};
  \pic
     at (480, 520) {ipe disk};
  \pic
     at (464, 504) {ipe disk};
  \pic
     at (480, 504) {ipe disk};
  \pic
     at (472, 536) {ipe disk};
  \pic
     at (400, 488) {ipe disk};
  \pic
     at (432, 488) {ipe disk};
  \pic
     at (336, 488) {ipe disk};
  \pic
     at (464, 488) {ipe disk};
  \pic
     at (352, 488) {ipe disk};
  \pic
     at (416, 488) {ipe disk};
  \draw
    (336, 784)
     -- (328, 800)
     -- (320, 784)
     -- (320, 768);
  \draw
    (352, 768)
     -- (352, 784)
     -- (352, 800);
  \draw
    (352, 800)
     -- (340, 816);
  \draw
    (328, 800)
     -- (340, 816);
  \draw
    (384, 768)
     -- (384, 784)
     -- (392, 800)
     -- (400, 784)
     -- (400, 768);
  \draw
    (416, 768)
     -- (416, 784)
     -- (416, 800);
  \draw
    (416, 800)
     -- (404, 816);
  \draw
    (392, 800)
     -- (404, 816);
  \draw
    (448, 768)
     -- (448, 784)
     -- (456, 800)
     -- (464, 784)
     -- (464, 768);
  \draw
    (480, 752)
     -- (480, 768)
     -- (480, 784)
     -- (480, 800);
  \draw
    (496, 768)
     -- (496, 784)
     -- (504, 800)
     -- (512, 784)
     -- (512, 768);
  \draw
    (528, 768)
     -- (528, 784)
     -- (536, 800)
     -- (544, 784);
  \draw
    (536, 800)
     -- (520, 816);
  \draw
    (504, 800)
     -- (520, 816);
  \draw
    (480, 800)
     -- (468, 816);
  \draw
    (456, 800)
     -- (468, 816);
  \draw
    (80, 680)
     -- (80, 696)
     -- (88, 712)
     -- (96, 696)
     -- (96, 680);
  \draw
    (112, 680)
     -- (112, 696)
     -- (120, 712)
     -- (128, 696)
     -- (128, 680);
  \draw
    (120, 712)
     -- (104, 728);
  \draw
    (88, 712)
     -- (104, 728);
  \draw
    (160, 664)
     -- (160, 680)
     -- (160, 696)
     -- (168, 712)
     -- (176, 696)
     -- (176, 680);
  \draw
    (192, 680)
     -- (192, 696)
     -- (200, 712)
     -- (208, 696)
     -- (208, 680);
  \draw
    (200, 712)
     -- (184, 728);
  \draw
    (184, 728)
     -- (168, 712);
  \draw
    (240, 664)
     -- (240, 680)
     -- (240, 696)
     -- (248, 712)
     -- (256, 696)
     -- (256, 680)
     -- (256, 664);
  \draw
    (272, 680)
     -- (272, 696)
     -- (280, 712)
     -- (288, 696)
     -- (288, 680);
  \draw
    (280, 712)
     -- (264, 728);
  \draw
    (264, 728)
     -- (248, 712);
  \draw
    (320, 680)
     -- (320, 696)
     -- (312, 712)
     -- (304, 696)
     -- (304, 680)
     -- (304, 664);
  \draw
    (336, 664)
     -- (336, 680)
     -- (336, 696)
     -- (344, 712)
     -- (352, 696)
     -- (352, 680);
  \draw
    (344, 712)
     -- (328, 728);
  \draw
    (328, 728)
     -- (312, 712);
  \draw
    (384, 664)
     -- (384, 680)
     -- (384, 696)
     -- (392, 712)
     -- (400, 696)
     -- (400, 680)
     -- (400, 664);
  \draw
    (416, 664)
     -- (416, 680)
     -- (416, 696)
     -- (424, 712)
     -- (432, 696)
     -- (432, 680);
  \draw
    (424, 712)
     -- (408, 728);
  \draw
    (408, 728)
     -- (392, 712);
  \draw
    (464, 664)
     -- (464, 680)
     -- (464, 696)
     -- (472, 712)
     -- (480, 696)
     -- (480, 680)
     -- (480, 664);
  \draw
    (496, 664)
     -- (496, 680)
     -- (496, 696)
     -- (504, 712)
     -- (512, 696)
     -- (512, 680)
     -- (512, 664);
  \draw
    (504, 712)
     -- (488, 728);
  \draw
    (488, 728)
     -- (472, 712);
  \draw
    (80, 592)
     -- (80, 608)
     -- (88, 624)
     -- (96, 608)
     -- (96, 592);
  \draw
    (112, 592)
     -- (112, 608)
     -- (120, 624)
     -- (128, 608)
     -- (128, 592);
  \draw
    (144, 592)
     -- (144, 608)
     -- (152, 624)
     -- (160, 608)
     -- (160, 592);
  \draw
    (120, 624)
     -- (120, 640);
  \draw
    (88, 624)
     -- (120, 640);
  \draw
    (120, 640)
     -- (152, 624);
  \draw
    (208, 576)
     -- (208, 592)
     -- (208, 608)
     -- (216, 624)
     -- (224, 608)
     -- (224, 592);
  \draw
    (240, 592)
     -- (240, 608)
     -- (248, 624)
     -- (256, 608)
     -- (256, 592);
  \draw
    (272, 592)
     -- (272, 608)
     -- (280, 624)
     -- (288, 608)
     -- (288, 592);
  \draw
    (280, 624)
     -- (248, 640);
  \draw
    (248, 624)
     -- (248, 640);
  \draw
    (216, 624)
     -- (248, 640);
  \draw
    (320, 576)
     -- (320, 592)
     -- (320, 608)
     -- (328, 624)
     -- (336, 608)
     -- (336, 592)
     -- (336, 576);
  \draw
    (368, 592)
     -- (368, 608)
     -- (360, 624)
     -- (352, 608)
     -- (352, 592);
  \draw
    (384, 592)
     -- (384, 608)
     -- (392, 624)
     -- (400, 608)
     -- (400, 592);
  \draw
    (416, 576)
     -- (416, 592)
     -- (416, 608)
     -- (424, 624)
     -- (432, 608)
     -- (432, 592);
  \draw
    (448, 576)
     -- (448, 592)
     -- (448, 608)
     -- (456, 624)
     -- (464, 608)
     -- (464, 592);
  \draw
    (480, 592)
     -- (480, 608)
     -- (488, 624)
     -- (496, 608)
     -- (496, 592);
  \draw
    (96, 488)
     -- (96, 504)
     -- (96, 520)
     -- (104, 536)
     -- (112, 520)
     -- (112, 504)
     -- (112, 488);
  \draw
    (128, 488)
     -- (128, 504)
     -- (128, 520)
     -- (136, 536)
     -- (144, 520)
     -- (144, 504);
  \draw
    (160, 504)
     -- (160, 520)
     -- (168, 536)
     -- (176, 520)
     -- (176, 504);
  \draw
    (192, 488)
     -- (192, 504)
     -- (192, 520)
     -- (200, 536)
     -- (208, 520)
     -- (208, 504);
  \draw
    (224, 488)
     -- (224, 504)
     -- (224, 520)
     -- (232, 536)
     -- (240, 520)
     -- (240, 504);
  \draw
    (256, 488)
     -- (256, 504)
     -- (256, 520)
     -- (264, 536)
     -- (272, 520)
     -- (272, 504);
  \draw
    (304, 488)
     -- (304, 504)
     -- (304, 520)
     -- (312, 536)
     -- (320, 520)
     -- (320, 504)
     -- (320, 488);
  \draw
    (336, 488)
     -- (336, 504)
     -- (336, 520)
     -- (344, 536)
     -- (352, 520)
     -- (352, 504)
     -- (352, 488);
  \draw
    (384, 504)
     -- (384, 520)
     -- (376, 536)
     -- (368, 520)
     -- (368, 504);
  \draw
    (400, 488)
     -- (400, 504)
     -- (400, 520)
     -- (408, 536)
     -- (416, 520)
     -- (416, 504)
     -- (416, 488);
  \draw
    (432, 488)
     -- (432, 504)
     -- (432, 520)
     -- (440, 536)
     -- (448, 520)
     -- (448, 504);
  \draw
    (464, 488)
     -- (464, 504)
     -- (464, 520)
     -- (472, 536)
     -- (480, 520)
     -- (480, 504);
  \draw
    (408, 536)
     -- (440, 552);
  \draw
    (440, 552)
     -- (440, 536);
  \draw
    (440, 552)
     -- (472, 536);
  \draw
    (312, 536)
     -- (344, 552);
  \draw
    (344, 552)
     -- (344, 536);
  \draw
    (344, 552)
     -- (376, 536);
  \draw
    (328, 624)
     -- (360, 640);
  \draw
    (360, 640)
     -- (360, 624);
  \draw
    (360, 640)
     -- (392, 624);
  \draw
    (424, 624)
     -- (456, 640);
  \draw
    (456, 640)
     -- (456, 624);
  \draw
    (456, 640)
     -- (488, 624);
  \draw
    (104, 536)
     -- (136, 552);
  \draw
    (136, 552)
     -- (136, 536);
  \draw
    (136, 552)
     -- (168, 536);
  \draw
    (232, 552)
     -- (200, 536);
  \draw
    (232, 552)
     -- (232, 536);
  \draw
    (232, 552)
     -- (264, 536);
\end{tikzpicture}
}
	\caption{The trees that minimize $\ExpeDProj$ for $n\le 20$. Their values of $\ExpeDProj$ are given in Table \ref{table:minimum_trees}. Roots are drawn atop each tree and each edge should be regarded as oriented away from the root.}
	\label{fig:minimum_trees}
\end{figure}
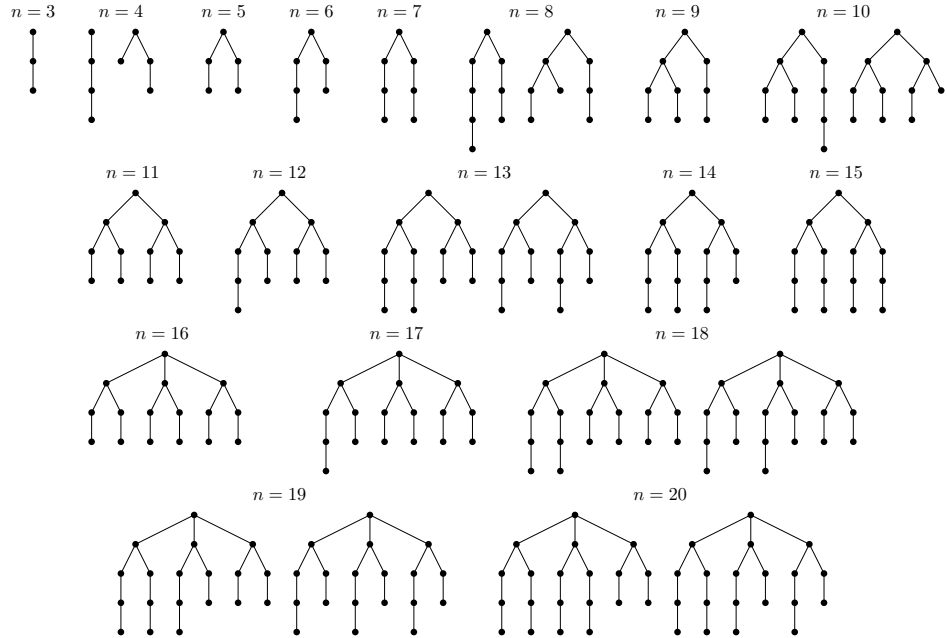

\begin{figure}
	\centering
	\includegraphics[scale=0.69]{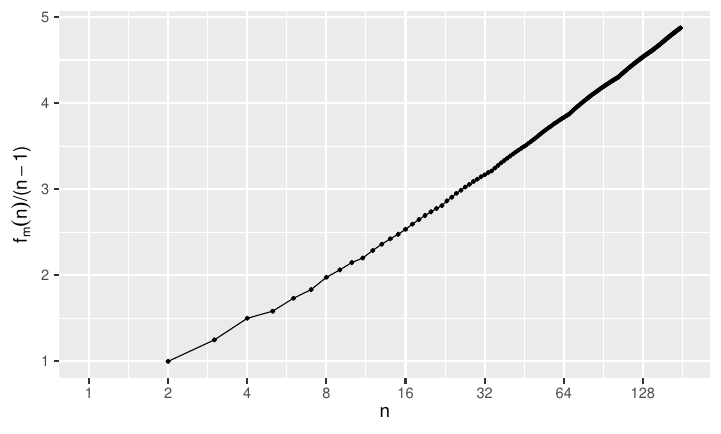}
	\caption{Values of $f_m(n)/(n - 1)$ in log-linear scale for $n\le178$, computed via Algorithm \ref{alg:dynamic_programming}.}
	\label{fig:values_minimum}
\end{figure}
\section{Exact versus approximate calculation of the expected sum of dependency distances}
\label{sec:exp_D:study_corpora}

In previous research the random baseline for $D$ under projectivity has been estimated with a Monte Carlo approximation method\footnote{This Monte Carlo method consists of averaging the values of $D$ obtained via random sampling of $\numsamples$ random projective arrangements of a rooted tree. It is well known that such methods' error decreases as the number of samples increases.} \cite{Gildea2007a,Park2009a, Gildea2010a,Futrell2015a,Kramer2021a} whose cost $\bigO{\numsamples n}$, where $\numsamples$ is the number of random arrangements sampled and $\bigO{n}$ is the cost of generating a random arrangement and computing $D$ on it. Our method, on top of providing exact values, has a much lower complexity $\bigO{n}$ as there is no need for random sampling. We used the {\em UD2.5} \cite{Web_UD} treebanks for Catalan, English and German to measure the relative error of the Monte Carlo approximation method for values of $\numsamples$ used in past research. English and German are selected because they have been used in previous research employing projective random baselines \cite{Gildea2007a,Park2009a,Gildea2010a}. Catalan is included as the native language of the present authors. For each sentence $\Rtree$ in a treebank, we calculated $\ExpeDProj$ in two ways. First, exactly with Equation \ref{eq:ExpeDProj:closed_form}. Second, approximately by averaging the value of $D$ obtained in $\numsamples=10^i$, for $1\le i\le 4$, uniformly random projective arrangements (denoted as $\ExpeDProjApprox$). More precisely, given $\{\arr_i\}_{i=1}^{\numsamples}$ random projective arrangements of a given rooted tree,
\begin{equation}
\ExpeDProjApprox = \frac{1}{\numsamples}\sum_{i=1}^{\numsamples} \D{\Rtree}[\arr_i]
\end{equation}
Using these values we calculated the relative error $\varepsilon_{rel}(\Rtree)$ for every tree as
\begin{equation}
\varepsilon_{rel}(\Rtree) =
    \frac{\ExpeDProjApprox - \ExpeDProj}{\ExpeDProj}
    \label{eq:relative_error}
\end{equation}
$\varepsilon_{rel}(\Rtree) > 0$ indicates that $\ExpeDProjApprox$ overestimates $\ExpeDProj$; $\varepsilon_{rel}(\Rtree) < 0$ indicates underestimation error. Figure \ref{fig:applications:lang_comp:relative_error} shows the  average, minimum, maximum and confidence interval of the relative error as a function of sentence length (see the Appendix for a parallel analysis of the standard deviation of $\VD{\Rtree}$). This figure can be seen as a confirmation of the correctness of Theorem \ref{thm:ExpeDProj} via simulation.

In previous research, the value of $\numsamples$ is sometimes indicated, for example, $\numsamples=10$ \cite{Kramer2021a} and $\numsamples=100$ \cite{Futrell2015a} and sometimes not reported \cite{Park2009a,Gildea2010a}. When $\numsamples=10$, Figure \ref{fig:applications:lang_comp:relative_error} shows that the relative error peaks between $n=10$ and $n=20$ (close to the mean sentence length in English, \citet{Rudnicka2018a}) and tends to decay from then onwards. However, the confidence interval of the relative error clearly broadens for sufficiently large $n$, for example, $n=30$ or $n=40$ onwards. These behaviors smooth out for $\numsamples=100$ while the range of variation and the confidence interval are narrower. As expected, the relative error reduces dramatically for larger $\numsamples$. In sum, Figure \ref{fig:applications:lang_comp:relative_error} indicates that a large $R$ must be used to estimate $\ExpeDProj$ with high numerical precision. However, that increases the computation time. An implication of that figure is that the best solution in terms of numerical precision and speed is provided by Equation \ref{eq:ExpeDProj:closed_form:n_minus_1}.

\begin{figure}
	\centering
	\includegraphics[height=488pt,width=380pt]{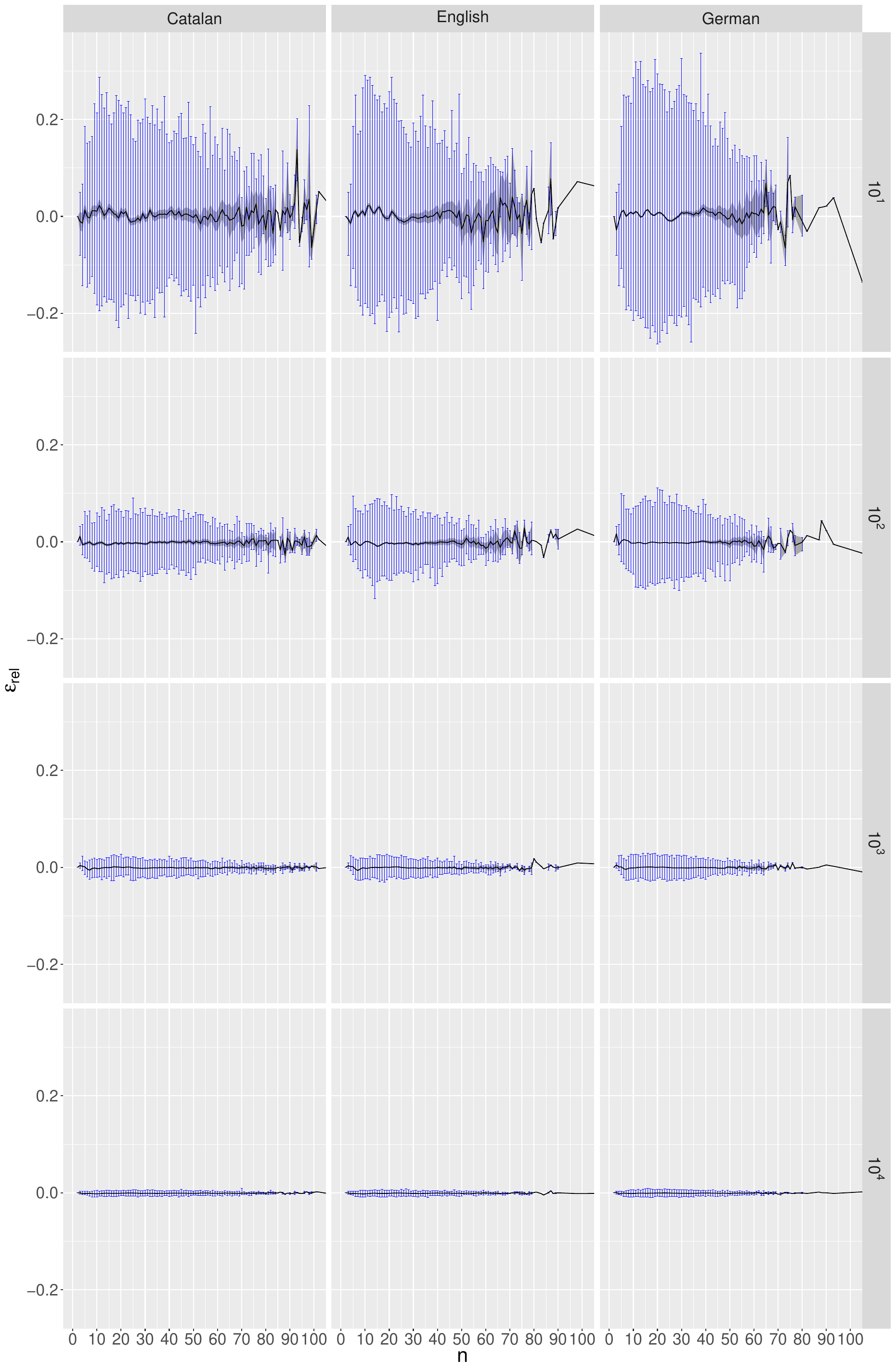}
	\caption{The statistical properties of $\varepsilon_{rel}$, the relative error in the estimation of $\ExpeDProj$ calculated with a Monte Carlo method (Equation \ref{eq:relative_error}), as a function of the size of the tree $n$. For all sentences of the same length, we show the average $\varepsilon_{rel}$ (solid black line), its $99\%$ confidence interval calculated with a bootstrap method (shaded gray region), and the maximum and minimum $\varepsilon_{rel}$ (vertical blue bars). The end of each row indicates $\numsamples$, that is the number of random projective arrangements used to estimate $\ExpeDProj$ (from $\numsamples=10$ for the top row to $\numsamples=10^4$ for the bottom row).}
	\label{fig:applications:lang_comp:relative_error}
\end{figure}
\section{Conclusions and future work}
\label{sec:conclusions}

In this article, we have derived several simple closed-form formulae related to projective arrangements of rooted trees in Section \ref{sec:exp_D}. Firstly, we have given a simple closed-form formula to calculate the amount of {\em different} projective arrangements a rooted tree admits, $\NProj$ (Proposition \ref{prop:NProj}). The proof reveals a straightforward way of enumerating such arrangements for an input rooted tree, and of generating this kind of arrangements uniformly at random without a rejection method. Secondly, and more importantly, we have provided a way of calculating, for any given rooted tree, the expected value of the sum of edge lengths over the space of uniformly random {\em different} projective arrangements, $\ExpeDProj$ (Theorem \ref{thm:ExpeDProj} and Corollary \ref{cor:ExpeDProj:reformulation}). This means that future studies in which such value is calculated approximately via random sampling of arrangements can now be calculated exactly and much faster. The $O(\numsamples n)$ Monte Carlo method to estimate true expectation with an error that tends to zero as $\numsamples$ tends to infinity can now be replaced by our fast $O(n)$ with zero error. Moreover, these formulae can be instantiated in particular classes of trees (as shown in Section \ref{sec:exp_D:tree_classes}). In Section \ref{sec:maxima_minima}, we have characterized the trees that maximize $\ExpeDProj$ and proven that there exists a dynamic programming method to calculate the minima of $\ExpeDProj$. A precise characterization of the minima should be the subject of future research. Finally, in Section \ref{sec:exp_D:study_corpora} we have highlighted the obvious advantages of an exact and fast calculation of $\ExpeDProj$ for future quantitative dependency syntax research.

The present article is a part of a research program on the calculation of random baselines for $D$ via formulae or exact algorithms under formal constraints on linear arrangements that started about two decades ago with the unconstrained case, for example Equation \ref{eq:ExpeDUnconstrained:closed_form} \cite{Zornig1984a,Ferrer2004a}. Here we have covered the projective case. In a forthcoming article, we will focus on planar linearizations of (free) trees -- those in which there are no edge crossings -- and obtain a closed-form formula, and a $\bigO{n}$-time algorithm, to calculate $\ExpeDPlan$, the expected value of $\VD{\Ftree}$ in uniformly random planar arrangements of a given (labeled free) tree $\Ftree$ \cite{Alemany2022b}. In the future, the problem of the calculation of the variance of $\VD{\Rtree}$ in random projective arrangements should also be considered. The analysis of the distribution of dependency distances, for example, their first and second moments of $\VD{\Rtree}$, in random arrangements that are not uniformly random (but still projective and planar), could benefit from applying general-purpose algorithmic frameworks \cite{Eisner2002a,Li2009a,Wang2018a}.

Our research paves the way to investigate the optimality of dependency distances of languages under projectivity. Recently, that optimality has been evaluated in 93 languages from 19 families with the help of a new score, $\Omega(\Ftree)$ that is defined with respect to the minimum and random baseline in unconstrained linear arrangements. $\Omega(\Ftree)$ is defined as \cite{Ferrer2022a}
\begin{equation}
\Omega(\Ftree) = \frac{\expe{\VD{\Ftree}} - \VD{\Ftree}}{\expe{\VD{\Ftree}} - \Dmin}
\end{equation}
where $\Dmin$ is the minimum $\D{\Ftree}$ over all unconstrained linear arrangements $\pi$, known as the Minimum Linear Arrangement problem \cite{Garey1976a,Shiloach1979a,Chung1984a}. However, projectivity is the most widely used constraint to investigate dependency distances and to define the corresponding minimum and random baselines \cite{Futrell2015a,Gulordava2015a,Futrell2020a}. With the result in Theorem \ref{thm:ExpeDProj}, one could replicate the aforementioned study under the projectivity constraint by redefining the score as
\begin{equation}
\Omega_{\mathrm{pr}}(\Rtree) = \frac{\ExpeDProj - \VD{\Rtree}}{\ExpeDProj - \DminProj}
\end{equation}
where the minimum sum of edge lengths under the projectivity constraint is denoted as $\DminProj$ \cite{Hochberg2003a,Gildea2007a}, and is linear-time computable \cite{Alemany2022a}. While $\Omega(\Ftree)\le 1$ holds for any sentence \cite{Ferrer2022a} as it is equivalent to $\VD{\Ftree}\ge\Dmin$, the statement $\Omega_{\mathrm{pr}}(\Rtree)\le 1$ only holds when applied to projective sentences. In other words, $\Omega_{\mathrm{pr}}(\Rtree)\le 1$ needs not hold since we can have that $\VD{\Rtree}>\DminProj$. In the absence of any word order constraint on a sentence, $\Omega$ is expected to be zero while $\Omega_{\mathrm{pr}}(\Rtree)$ is expected to be zero if the only word order constraint is projectivity. Formally, $\expe{\Omega(\Ftree)}=0$ and $\rexpe{\Omega_{\mathrm{pr}}(\Rtree)}=0$. Thanks to our article, such investigation of the optimality of dependency distances can be carried out reducing the computational cost and maximizing numerical precision with respect to an approach based on a Monte Carlo estimation of $\ExpeDProj$.

Here we have focused on $\rexpe{\VD{\Rtree}}$, the expectation of $\VD{\Rtree}$ on random arrangements of an individual tree. Finally, future research should consider the problem of $\expe{\rexpe{D}}$, the expectation of $\rexpe{D}$ on ensembles of random trees of a fixed size $n$. $\expe{\rexpe{D}}$ is indeed the average value of $\ExpeDProj$ among all $n$-vertex rooted trees $\Rtree$. Although there are at least two ensembles possible, that is uniformly random {\em labeled} trees and uniformly random {\em unlabeled} trees, a closed-form formula (or algorithm) to calculate $\expe{\rexpe{D}}$ seems easier to obtain in the former.

\begin{acknowledgments}
We are thankful to the anonymous reviewers for invaluable comments and suggestions. LAP is supported by Secretaria d'Universitats i Recerca de la Generalitat de Catalunya and the Social European Fund. RFC is also supported by the recognition 2017SGR-856 (MACDA) from AGAUR (Generalitat de Catalunya). RFC and LAP are supported by the grant TIN2017-89244-R from MINECO (Ministerio de Econom\'ia, Industria y Competitividad).
\end{acknowledgments}

\appendix
\appendixsection{The standard deviation of $\VD{\Rtree}$ in projective arrangements}
\label{sec:appendix:standard_deviation}

Besides estimating the relative error of a Monte Carlo method to approximate $\ExpeDProj$, we also calculated the standard deviation $\sigma(\Rtree)$ of $\D{\Rtree}[\arr]$ over $R$ random projective arrangements $\arr$ via
\begin{equation}
\label{eq:standard_deviation}
\sigma^2(\Rtree) =
	\frac{1}{\numsamples - 1}
	\sum_{i=1}^{\numsamples}
	\left(
		\D{\Rtree}[\arr_i] - \ExpeDProj
	\right)^2
\end{equation}
Notice that $\sigma$ is calculated using $\ExpeDProj$ and not $\ExpeDProjApprox$. Figure \ref{fig:applications:lang_comp:variance} shows that $\sigma$ increases with $n$ (as expected) but that it does not decrease as $\numsamples$ increases; furthermore, the apparent linear trend for sufficiently large $n$ suggests an $\bigO{n}$ growth of $\sigma$ to be confirmed in future research.

\begin{figure}
	\centering
	\includegraphics[height=488pt,width=380pt]{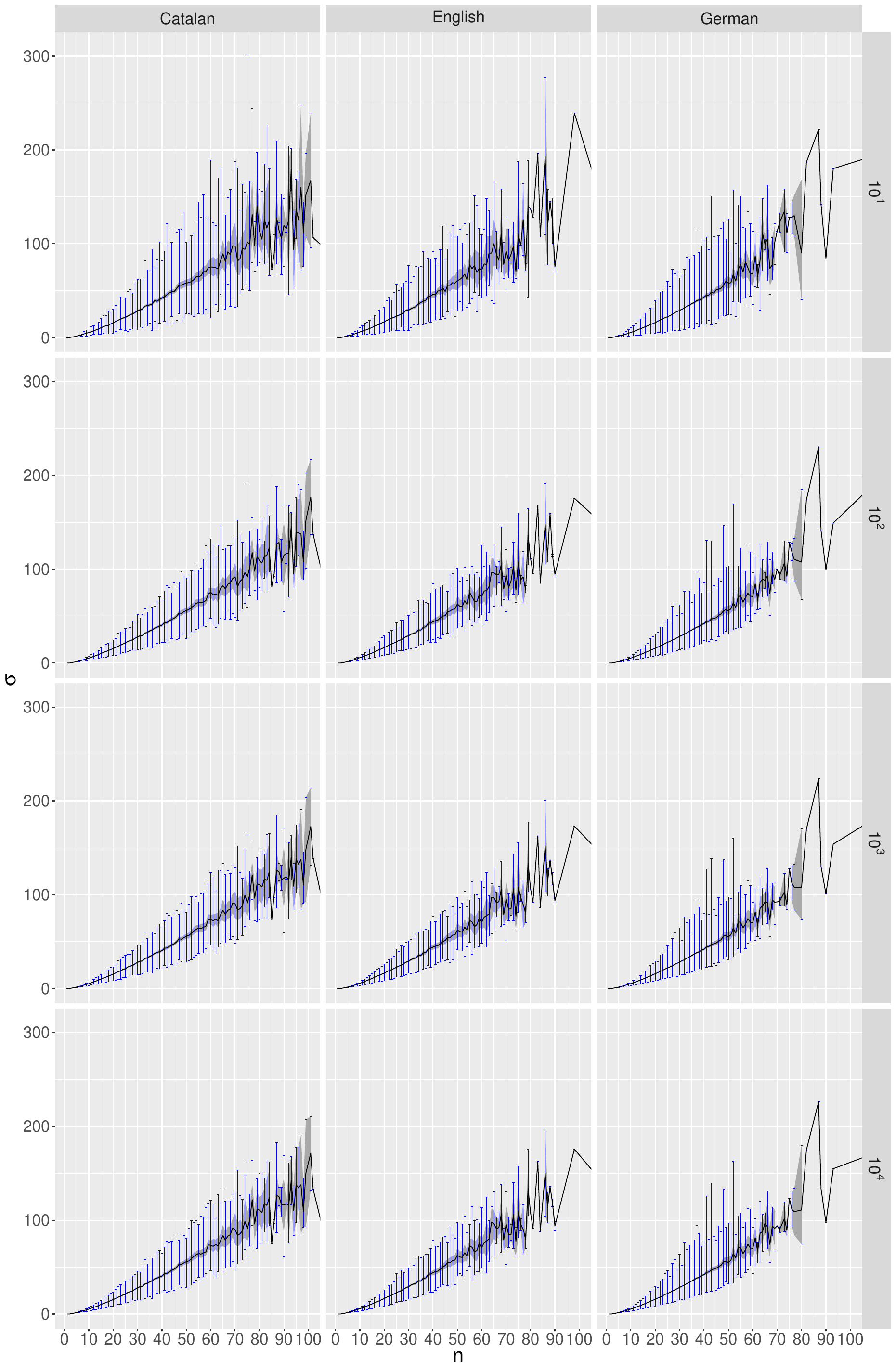}
	\caption{The statistical properties of $\sigma$, the standard deviation in the estimation of $\ExpeDProj$ calculated with a Monte Carlo method (Equation \ref{eq:standard_deviation}), as a function of the size of the tree $n$. This figure has the same format as Figure \ref{fig:applications:lang_comp:relative_error}. For all sentences of the same length, we show the average $\sigma$ (solid black line), its $99\%$ confidence interval calculated with a bootstrap method (shaded gray region), and the maximum and minimum $\varepsilon_{rel}$ (vertical blue bars). The end of each row indicates $\numsamples$, that is the number of random projective arrangements used to estimate $\ExpeDProj$ (from $\numsamples=10$ for the top row to $\numsamples=10^4$ for the bottom row).}
	\label{fig:applications:lang_comp:variance}
\end{figure}

\starttwocolumn

\end{document}